\documentclass[twoside]{article}

\pdfoutput=1

\usepackage[accepted]{aistats2021}

\usepackage[hidelinks]{hyperref}

\usepackage{color}
\usepackage{bm}
\usepackage{natbib}
\usepackage{xspace}
\usepackage{amsthm}
\usepackage[ruled,vlined]{algorithm2e} 
\usepackage{cleveref}
\usepackage{booktabs}
\usepackage{url}
\newtheorem{lemma}{Lemma}
\newtheorem{theorem}{Theorem}
\newtheorem{corollary}{Corollary}
\newtheorem{proposition}{Proposition}
\newtheorem{assumption}{Assumption}

\usepackage{natbib}
\newcommand{\REPS}{\texttt{REPS}\xspace}
\newcommand{\QREPS}{\texttt{Q-REPS}\xspace}

\newcommand{\OPT}{\texttt{MinMax-Q-REPS}\xspace}
\newcommand{\QREPSEVAL}{\texttt{Q-REPS-Eval}\xspace}

\newcommand{\qq}{\mu}
\newcommand{\qqd}{d}
\newcommand{\dd}{d}

\newcommand{\tq}{\wt{\qq}}
\newcommand{\td}{\wt{\dd}}

\renewcommand{\SS}{\mathcal{S}}

\newcommand{\wh}{\widehat}
\newcommand{\wt}{\widetilde}

\newcommand{\X}{\mathcal{X}}

\newcommand{\A}{\mathcal{A}}

\newcommand{\LL}{\mathcal{L}}
\newcommand{\real}{\mathbb{R}}

\newcommand{\GG}{\mathcal{G}}
\newcommand{\MM}{\mathcal{M}}
\newcommand{\UU}{\mathcal{U}}
\newcommand{\DD}{\mathcal{D}}
\newcommand{\D}[2]{D\left(#1\middle\|#2\right)}

\newcommand{\QQ}{\mathcal{Q}}
\newcommand{\NN}{\mathcal{N}}
\newcommand{\CC}{\mathcal{C}}

\newcommand{\II}[1]{\mathbb{I}_{\left\{#1\right\}}}

\newcommand{\EE}[1]{\mathbb{E}\left[#1\right]}

\newcommand{\EEs}[2]{\mathbb{E}_{#2}\left[#1\right]}

\newcommand{\EEc}[2]{\mathbb{E}\left[#1\left|#2\right.\right]}

\def\argmin{\mathop{\mbox{ arg\,min}}}
\def\argmax{\mathop{\mbox{ arg\,max}}}
\newcommand{\ra}{\rightarrow}

\newcommand{\bone}{\bm{1}}

\newcommand{\iprod}[2]{\left\langle#1,#2\right\rangle}
\newcommand{\biprod}[2]{\bigl\langle#1,#2\bigr\rangle}

\newcommand{\norm}[1]{\left\|#1\right\|}
\newcommand{\onenorm}[1]{\norm{#1}_1}

\newcommand{\infnorm}[1]{\norm{#1}_\infty}

\newcommand{\ev}[1]{\left\{#1\right\}}
\newcommand{\pa}[1]{\left(#1\right)}
\newcommand{\bpa}[1]{\bigl(#1\bigr)}

\newcommand{\bW}{\overline{W}}

\newcommand{\tnu}{\wt{\nu}}

\newcommand{\transpose}{^\mathsf{\scriptscriptstyle T}}

\usepackage{todonotes}
\definecolor{PalePurp}{rgb}{0.66,0.57,0.66}

\usepackage{amsfonts}

\newcommand{\vtet}{\theta}

\newcommand{\dual}{\mathcal{G}}
\newcommand{\ymu}{z}

\newcommand{\Vtet}{V_\vtet}

\allowdisplaybreaks

\begin{document}

\twocolumn[

\aistatstitle{Logistic Q-Learning}

\aistatsauthor{Joan Bas-Serrano \And Sebastian Curi \And Andreas Krause \And  Gergely Neu }

\aistatsaddress{\small{Universitat Pompeu Fabra} \And  \small{ETH Z\"urich} \And \small{ETH Z\"urich} \And 
\small{Universitat Pompeu Fabra}  } ]

\begin{abstract}
We propose a new reinforcement learning algorithm derived from a regularized linear-programming formulation of 
optimal control in MDPs. The method is closely related to the classic Relative Entropy Policy Search (\REPS) algorithm 
of \citet{peters10reps}, with the key difference that our method introduces a Q-function that enables efficient 
exact model-free implementation. The main feature of our algorithm (called \QREPS) is a convex loss 
function for policy evaluation that serves as a theoretically sound alternative to the widely used squared Bellman 
error.  We provide a practical saddle-point optimization method for minimizing this loss function and provide an 
error-propagation analysis that relates the quality of the individual updates to the performance of the output policy. 
Finally, we demonstrate the effectiveness of our method on a range of benchmark problems. 
\end{abstract}

\section{INTRODUCTION}
While the squared Bellman error is a broadly used loss function for approximate dynamic programming and reinforcement learning (RL), it has a number of undesirable properties: it is not directly motivated by standard Markov Decission Processes (MDP) theory, not convex in the action-value function parameters, and RL  algorithms based on its recursive optimization are known to be unstable 
\citep{GPP17,MM20}. 
In this paper, we offer a remedy to these issues by proposing a new RL algorithm utilizing an objective-function free from these problems. 
Our approach is based on the seminal Relative Entropy Policy Search (\REPS) algorithm of \citet*{peters10reps}, with a number of newly introduced elements that make the algorithm significantly more practical.

While \REPS is elegantly derived from a principled linear-programing (LP) formulation of optimal control in MDPs, it 
has the serious shortcoming that its faithful implementation requires access to the true MDP for both the policy 
evaluation and 
improvement steps, even at deployment time. The usual way to address this limitation is to use an empirical 
approximation to the policy evaluation step and to project the policy from the improvement step into a parametric space 
\citep{deisenroth2013survey}, losing all the theoretical guarantees of \REPS in the process.

In this work, we propose a new algorithm called \QREPS that eliminates this limitation of \REPS by introducing a simple 
softmax policy improvement step expressed in terms of an action-value function that naturally arises from a regularized LP formulation. 
The action-value functions are obtained by minimizing a convex loss function that we call the \emph{logistic Bellman error} (LBE) due to its analogy with the classic notion of Bellman error and the logistic loss for logistic regression. 
The LBE has numerous advantages over the most commonly used notions of Bellman error: unlike the squared Bellman error, 
the logistic Bellman error is convex in the action-value function parameters, smooth, and has bounded gradients (see 
Figure~\ref{fig:square_bellman_error}). 
This latter property obviates the need for the heuristic technique of gradient clipping (or using the Huber loss in 
place of the square loss), a commonly used optimization trick to improve stability of training of deep RL algorithms 
\citep{mnih2015human}.

Besides the above favorable properties, \QREPS comes with rigorous theoretical guarantees that establish its 
convergence to the optimal policy under appropriate conditions. Our main theoretical contribution is an 
error-propagation analysis that relates the quality of the optimization subroutine to the quality of the policy output 
by the algorithm, showing that convergence to the optimal policy can be guaranteed if the optimization errors are kept 
sufficiently small. Together with another result that establishes a bound on the bias of the empirical LBE in terms of 
the regularization parameters used in \QREPS, this justifies the approach of minimizing the empirical objective under 
general conditions. For the concrete setting of factored linear MDPs, we provide a bound on the rate of 
convergence. 

Our main algorithmic contribution is a saddle-point optimization framework for optimizing the empirical version of the 
LBE that formulates the minimization problem as a two-player game between a \emph{learner} and a \emph{sampler}. The 
learner plays stochastic gradient descent (SGD) on the samples proposed by the sampler, and the sampler updates its 
distribution over the sample transitions in response to the observed Bellman errors. We evaluate the resulting 
algorithm experimentally on a range of standard benchmarks, showing excellent empirical performance of \QREPS.

\begin{figure}[t]
\includegraphics[width=\columnwidth]{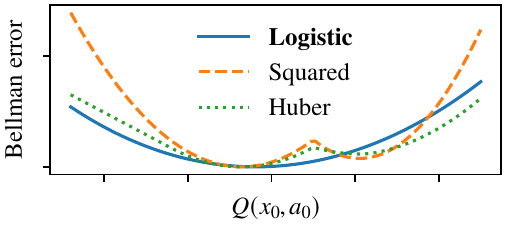} 
\caption{Squared Bellman error considered harmful: Loss functions plotted as a function of the Q-value at a 
fixed state-action pair while keeping other values fixed. 
} \label{fig:square_bellman_error}
\end{figure}

\vspace{-.2cm}
\paragraph{Related Work.}
Despite the enormous empirical successes of deep reinforcement learning, we understand little about the convergence of 
the algorithms that are commonly used. The use of the squared Bellman error for deep reinforcement learning has been 
popularized in 
the breakthrough paper of \citet{mnih2015human}, and has been \emph{exclusively} used for policy evaluation ever since. 
Indeed, while several algorithmic improvements have been proposed for improving policy updates over the past few years, 
the squared Bellman error remained a staple: among others, it is used for policy evaluation in TRPO~\citep{SLAJM15}, 
SAC~\citep{haarnoja2018soft}, A3C~\citep{M+16}, TD3~\citep{fujimoto2018addressing}, MPO~\citep{abdolmaleki2018maximum} 
and POLITEX~\citep{abbasi2019politex}. Despite its extremely broad use, the squared Bellman error suffers from a range 
of well-known issues pointed out by several authors including \citet[Chapter 11.5]{SB18}, \citet{GPP17}, and 
\citet{MM20}. While some of these have been recently addressed by \citet{dai2018sbeed} and \citet{FLL19}, several 
concerns remain.

On the other hand, the RL community has been very productive in developing novel policy-improvement rules: since the 
seminal work of \citet{kakade2002approximately} established the importance of soft policy updates for dealing with 
policy-evaluation errors, several practical update rules have been proposed and applied successfully in the context of 
deep RL---see the list we provided in the previous paragraph. Many of these soft policy updates are based on the idea 
of \emph{entropy regularization}, first explored by \citet{K01} and \citet{ziebart08MaxEntIRL} and inspiring 
an impressive number of followup works eventually unified by \citet{NJG17} and \citet{geist2019theory}. A particularly 
attractive feature of entropy-regularized methods is that they often come with a closed-form ``softmax'' policy update 
rule that is easily expressed in terms of an action-value function. A limitation of these methods is that they 
typically don't come with a theoretically well-motivated loss function for estimating the value functions and end up 
relying on the squared Bellman error. One notable exception is the REPS algorithm of \citet{peters10reps} that comes with a 
natural loss function for policy evaluation, but no tractable policy-update rule.

The main contribution of our work is proposing \QREPS, a mirror-descent algorithm that comes with both a natural loss 
function and an explicit and tractable policy update rule, both derived from an entropy-regularization perspective. 
These properties make it possible to implement \QREPS \emph{entirely faithfully} to its theoretical specification in a 
deep reinforcement learning context, modulo the step of using a neural network for parametrizing the Q function. This 
implementation is justified by our main theoretical result, an error propagation analysis accounting for the 
optimization and representation errors.

Our error propagation analysis is close in spirit to that of \citet{scherrer2015approximate}, recently extended to 
entropy-regularized approximate dynamic programming  algorithms by \citet{geist2019theory}, 
\citet{vieillard2020leverage}, and \citet{vieillard2020munchausen}. 
One major difference between our approaches is that their guarantees depend on the $\ell_p$ norms of the policy 
evaluation errors, but still optimize squared-Bellman-error-like quantities that only serve as proxy for these errors. 
In contrast, our analysis studies the propagation of the optimization errors on the objective function that is 
\emph{actually optimized} by the algorithm.

\vspace{-.19cm}
\paragraph{Notation.} We use $\iprod{\cdot}{\cdot}$ to denote inner products in Euclidean space and $\real_+$ to denote 
the set of non-negative real numbers. For two vectors $v,w\in\real^m$, we will use the notation $v\ge w$ to denote 
elementwise inequality holding in the sense $v-w\in\real_+^m$.
We will often write indefinite sums $\sum_{x,a}$ to denote sums over 
the 
entire state-action space $\X\times\A$, and write $p(x,a)\propto q(x,a)$ to signify that $p(x,a) = q(x,a) / 
\sum_{x',a'} q(x',a')$ for a nonnegative function $q$ over $\X\times\A$.

\section{BACKGROUND}
Consider a Markov decision process (MDP, \citealp{Puterman1994}) defined by the tuple $M = (\mathcal{X}, 
\mathcal{A},  P, r)$, where  $\mathcal{X}$ is the state space, $\mathcal{A}$ is the 
action space, $P$ is the transition function with $P(\cdot|x,a)$ denoting the distribution of the follow-up state
$x'$ after taking action $a\in\mathcal{A}$ in state $x\in\mathcal{X}$,
and $r$ is the reward function mapping state-action pairs to rewards 
with $r(x,a)$ denoting the reward of being in state $x$ and taking action $a$. For simplicity of presentation, 
we assume that the rewards are deterministic and bounded in $[0,1]$, and that the state action spaces are finite (but 
potentially very large). An MDP models a sequential interaction process between an agent and its environment 
where in each round $t$, the agent observes state  $x_t\in\X$, selects action $a_t\in\mathcal{A}$, moves to the next 
state $x_{t+1} \sim P(\cdot|x_t,a_t)$, and obtains reward $r(x_t, a_t)$. 
The goal of the agent is to select actions so as to maximize the \emph{normalized discounted 
return} $R = (1-\gamma)\mathbb{E}\left[\sum_{t=0}^\infty \gamma^t r(x_t,a_t)\right]$, where $\gamma\in(0,1)$ is the 
discount factor and the state $x_0$ is drawn from a fixed initial-state distribution $\nu_0$.

We will heavily rely on a \emph{linear programming} (LP) characterization of optimal policies originally due to 
\citealp{Man60}. This approach aims to directly find a 
\emph{normalized discounted state-action occupancy measure} (in short, \emph{occupancy measure}) $\qq(x,a) = 
(1-\gamma)\EE{\sum_{t=0}^\infty \gamma^t \II{(x_t,a_t) = (x,a)}}$ with $x_0 \sim \nu_0$ that maximizes the discounted 
return that can simply be written as $R = \sum_{x,a} \qq(x,a) r(x,a)$. 
From every valid occupancy measure $\qq$, one can derive a \emph{stationary stochastic policy} (in short, 
\emph{policy}) $\pi_\qq$ defined as the conditional distribution $\pi_\qq(a|x) = \qq(x,a) / \sum_{a'} \qq(x,a')$ over 
actions $a$ for each state $x$. Following the policy $\pi_\qq$ by drawing each action $a_t \sim \pi(\cdot|x_t)$ can be 
shown to yield $\qq$ as the occupancy measure. We briefly describe the characterization of optimal policies in 
these terms below, and refer the interested reader to Section~6.9 of \citet{Puterman1994} for a 
more detailed discussion.

For a compact notation, we will represent the decision variables $\qq$ as vectors in $\real^{\X\times\A}$ and 
introduce 
the linear operator $P\transpose:\real^{\X\times\A}\ra\real^{\X}$ defined for each $\qq$ through $(P\transpose \qq)(x') 
= \sum_{x,a} P(x'|x,a) \qq(x,a)$ for all $x'$. Similarly, we define the operator $E\transpose$ acting on $\qq$ through 
the 
assignment $(E\transpose \qq)(x) = \sum_{a}  \qq(x,a)$ for all $x$. With this notation, the task of finding an optimal 
occupancy measure can be written as the solution of the following linear program:
\begin{equation}\label{eq:primalLP}
\begin{split}
 \text{maximize}_{\qq \in \real^{\X\times\A}_+} \quad &\iprod{\qq}{r}
 \\
 \text{s.t.} \quad & E\transpose \qq = \gamma P\transpose \qq + (1-\gamma) \nu_0.
\end{split}
\end{equation}
The above set of constraints is known to uniquely characterize the set of all valid occupancy measures, which set will 
be denoted as $\MM^*$ from here on. Due to this property, any solution $\qq^*$ of the LP maximizes the total discounted 
return and the corresponding policy $\pi^* = \pi_{\qq^*}$ is optimal in the sense that choosing actions as 
$a_t\sim\pi^*(\cdot|x_t)$ yields maximal return.
The dual of the linear program~\eqref{eq:primalLP} is
\begin{equation}\label{eq:dualLP}
\begin{split}
 \text{minimize}_{V\in\real^\X} \quad & (1-\gamma) \iprod{\nu_0}{V}
 \\
 \text{s.t.} \quad & EV \ge r + \gamma PV,
\end{split}
\end{equation}
where we used the adjoint operators $E$ and $P$, acting on $V$ as $(EV)(x,a) = V(x)$ and $(PV)(x,a) = \sum_{x'} 
P(x'|x,a)V(x')$ for all $x,a$. The solution of this LP can be shown to be equivalent to the celebrated Bellman 
optimality equations in the sense that the so-called \emph{optimal value function} $V^*$ is an optimal solution of this 
LP, and is the unique optimal solution if $\nu_0$ has full support over the state space.

\paragraph{Relative Entropy Policy Search.} Our approach is directly inspired by  the seminal \emph{relative entropy 
policy search} (\REPS) algorithm proposed by \citet*{peters10reps}. The core ideas underlying \REPS are adding a 
strongly convex regularization function to the objective of the LP~\eqref{eq:primalLP} and relaxing the primal 
constraints through the use of a feature map $\psi:\X\ra\real^m$. Introducing the operator $\Psi\transpose$ acting on 
$q\in\real^\X$ as $\Psi\transpose q = \sum_{x} q(x) \psi(x)$, and letting $\mu_0$ be an arbitrary state-action 
distribution, \REPS is defined as an iterative optimization scheme that produces a sequence of occupancy measures as 
follows:
\begin{equation}\label{eq:REPS_OP}
\begin{split}
 \qq_{k+1} = \max_{\qq \in \real^{\X\times\A}_+} \quad &\iprod{\qq}{r} - \frac{1}{\eta} \D{\qq}{\qq_{k}}
 \\
 \text{s.t.} \quad & \Psi\transpose E\transpose \qq = \Psi\transpose \pa{\gamma P\transpose \qq + (1-\gamma)\nu_0}.
\end{split}\raisetag{1.3cm}
\end{equation}
Here, $\D{\qq}{\qq'}$ is the \emph{unnormalized relative entropy} (or Kullback--Leibler divergence) between the 
distributions $\qq$ and 
$\qq'$ defined as
$\D{\qq}{\qq'} = \sum_{x,a} \bpa{\qq(x,a) \bpa{\log \frac{\qq(x,a)}{\qq'(x,a)} - 1} + \qq'(x,a)}$.
Introducing the 
notation $V_\theta = \Psi\theta$ and 
$\delta_\theta = r + \gamma PV_\theta - EV_\theta$, the unique optimal solution to this optimization problem can be 
written as
\begin{equation}\label{eq:REPS-update}
 \qq_{k+1}(x,a) = \qq_{k}(x,a) e^{\eta \pa{\delta_{\theta_k}(x,a) - \rho_k}},
\end{equation}
where $\rho_k$ is a normalization constant and $\theta_k$ is given as the minimizer of the \emph{dual function} given as
\begin{equation}\label{eq:REPS-dual}
 \mathcal{G}_k(\theta) =\sum_{x,a} \qq_{k}(x,a) e^{\eta\delta_\theta (x,a)} + (1-\gamma) \iprod{\nu_0}{V_\theta}.
\end{equation}
As highlighted by \citet{ZiNe13} and \citet{NJG17}, REPS can be seen as a \emph{mirror descent} algorithm
\citep{Martinet1970,R76,BT03}, and thus its iterates $\qq_k$ are 
guaranteed to converge to an optimal occupancy measure $\mu^*$.

Despite its exceptional elegance, the formulation above has a number of features that limit its practical 
applicability. One very serious limitation of REPS is that its output policy $\pi_{K}$ involves an expectation with respect to the transition function, thus requiring knowledge of $P$ to 
run the policy. Another issue is that optimizing an empirical version of the loss~\eqref{eq:REPS-dual} as originally 
proposed by 
\citet{peters10reps} may be problematic due to the empirical loss being a biased estimator of the true 
objective~\eqref{eq:REPS-dual} caused by the conditional expectation appearing in the exponent.

\paragraph{Deep $Q$-learning.} Let us contrast REPS with the emblematic deep RL approach of Deep Q Networks 
(DQN) as proposed by \citet{mnih2015human}. This algorithm aims to approximate the \emph{optimal action-value function} 
$Q^*(x,a)$ which is known to characterize optimal behaviors: any policy that puts all probability mass on $\argmax_a 
Q^*(x,a)$ is optimal. Using the notation $\norm{f}_\qq^2 = \sum_{x,a} \qq(x,a) f^2(x,a)$, the main 
idea of DQN is to sequentially compute approximations of $Q^*$ by 
minimizing the \emph{squared Bellman error}:
\begin{equation}\label{eq:sqBellman}
Q_{k+1} = \argmin_{Q\in\QQ} \norm{r + \gamma P V_Q - Q}^2_{\qq_{k}},
\end{equation}
where $\QQ$ is some class of action-value functions (e.g., a class of neural networks), 
$V_Q(x) = \max_a Q(x,a)$, and $\qq_{k}$ is the state distribution generated by the policy $\pi_{k}$. A major 
advantage of this formulation is that, having access to Q-functions, it is trivial to compute policy updates, typically 
by choosing near-greedy policies with respect to $Q_k$. However, it is well known 
that the squared Bellman error objective above suffers from a number of serious problems: its 
lack of convexity in $Q$ prevents efficient optimization even under the simplest parametrizations, 
and the conditional expectation appearing within the squared norm makes its empirical estimate severely biased.

\citet{mnih2015human} addressed these issues by using a number of ideas from the approximate dynamic programming 
literature (see, e.g., \citealp{riedmiller2005neural}), eventually resulting in 
spectacular empirical performance on a range of highly challenging problems. Despite these successes, the heuristics 
introduced to stabilize DQN training are arguably only surface-level patches: the convergence of the resulting scheme 
can only be guaranteed under extremely strong conditions on the function class $\QQ$ and the data-generating 
distribution 
\citep{melo2007q,szepes06learning,GPP17,fan2020theoretical,MM20}. Altogether, these observations  suggest that the 
squared Bellman error has fundamental limitations that have to be addressed from first principles.

\paragraph{Our contribution.} In this paper, we address the above issues by proposing a new algorithmic framework that 
unifies the advantages of \REPS and DQNs, while removing their key limitations. Our approach (called \QREPS) endows 
\REPS with a Q-function fully specifying the policy updates, thus enabling efficient model-free implementation akin to 
DQNs. Similarly to \REPS, the Q-functions of \QREPS are obtained by minimizing a convex objective function (that we 
call \emph{logistic Bellman error}) naturally derived from a regularized LP formulation. 
We provide a practical framework for optimizing this objective and provide formal performance guarantees for 
the resulting algorithm.

\section{Q-REPS}
This section presents our main contributon: the derivation of the \QREPS algorithm in its primal and dual forms, and an 
efficient reinforcement learning algorithm that approximately implements the \QREPS policy updates using sample 
transitions.

One key technical idea underlying our algorithm design is a \emph{Lagrangian decomposition} of the linear 
program~\eqref{eq:primalLP}. Specifically, we introduce an additional set of primal variables $d \in 
\real^{\X\times\A}$ and split the constraints of the LP as follows:
\begin{equation}\label{eq:primalLP_decomp}
\begin{split}
 \text{maximize}_{\qq,d} \quad &\iprod{\qq}{r}
 \\
 \text{s.t.} \quad & E\transpose d = \gamma P\transpose \qq + (1-\gamma) \nu_0
 \\
 & \,\quad d = \qq, \qquad \qq\in\real^{\X\times\A}, d\in\real_+^{\X\times\A}.
\raisetag{1.5cm}
\end{split}
\end{equation}
The additional set of variables $\dd$ can be thought of as a ``mirror image'' of $\qq$.
By straightforward calculations, the dual of this LP can be shown to be
\begin{equation}\label{eq:dualLP_decomp}
\begin{split}
 \text{minimize}_{V\in\real^\X,Q\in\real^{\X\times\A}} \quad & (1-\gamma) \iprod{\nu_0}{V}
 \\
 \text{s.t.} 
 \quad 
 & Q = r + \gamma PV, \qquad  EV \ge Q.\raisetag{1.1cm}
\end{split}
\end{equation}
The optimal value functions $V^*$ and $Q^*$ can be easily seen to be optimal solutions of this decomposed LP. A 
clear advantage of this formulation that we will take advantage of is that it naturally introduces Q-functions as slack 
variables enforcing to the newly introduced primal constraints $d=\mu$.
To our best knowledge, this LP has been first proposed by \citet{mehta2009q} and has been recently rediscovered by 
\citet{lee2019stochastic} and \citet{NPB20} and revisited by \citet{MM20}.

Inspired by \citet{peters10reps}, we make two key modifications to this LP to derive our algorithm: introduce a convex 
regularization term in the objective and relax some of the constraints. For this latter step, we introduce a 
state-action feature map $\varphi:\real^{\X\times\A}\ra\real^m$ and the corresponding linear operator $\Phi\transpose$ 
acting on $\qq$ as $\Phi\transpose \qq = \sum_{x,a} \qq(x,a) \varphi(x,a)$. 
Further, we propose to augment the relative-entropy regularization used in \REPS by a 
\emph{conditional relative entropy} term defined between two state-action 
distributions $d$ and $d'$ as $H(d\|d') = \sum_{x,a} d(x,a) \log\frac{\pi_d(a|x)}{\pi_{d'}(a|x)}$. A minor change is 
that we will restrict $d$ and $\qq$ to belong to the set of 
probability distributions over $\X\times\A$, denoted as $\UU$.

Letting $\qq_0$ and $d_0$ be two arbitrary reference distributions and denoting the corresponding policy as $\pi_0 = 
\pi_{d_0}$, and letting $\alpha$ and $\eta$ be two positive parameters, we define the primal \QREPS optimization 
problem as follows:
\begin{equation}\label{eq:QREPS_OP}
\begin{split}
 \text{maximize}_{\qq,d\in\UU}& \quad \iprod{\qq}{r} - \frac{1}{\eta} D(\qq\|\qq_0) - \frac{1}{\alpha} H(d\|d_0)
 \\
 \text{s.t.} \quad & E\transpose d = \gamma P\transpose \qq + (1-\gamma) \nu_0
 \\
 & \Phi\transpose d = \Phi\transpose \qq.
\end{split}\raisetag{.99cm}
\end{equation}
The following proposition characterizes the optimal solution of this problem.
\begin{proposition}\label{prop:QREPS-structure}
 Define the Q-function $Q_\theta = \Phi\theta$ taking values $Q_\theta(x,a) = \iprod{\theta}{\varphi(x,a)}$, the value 
function
 \begin{equation}\label{eq:value}
  V_\theta(x) = \frac1{\alpha}\log\left(\sum_a  \pi_0(a|x) e^{\alpha Q_\theta(x,a)}\right)
 \end{equation}
and the Bellman error function $\Delta_\theta = r + \gamma P V_\theta - Q_\theta$. 
Then, the optimal solution of the optimization problem~\eqref{eq:QREPS_OP} is given as
\begin{align*}
 \qq^*(x,a) &\propto \qq_0(x,a) e^{\eta \Delta_{\theta^*}(x,a)}
 \\
 \pi_{d^*}(a|x) &= \pi_0(a|x) e^{\alpha \bpa{Q_{\theta^*}(x,a) - V_{\theta^*}(x)}}, 
\end{align*}
where $\theta^*$ is the minimizer of the convex function
\begin{align*}\label{eq:QREPS-dual}
 \mathcal{G}(\theta) =&\frac1{\eta} \log \left(\sum_{x,a}  {\qq_0(x,a) e^{\eta \Delta_\theta(x,a)}}\right) + 
(1-\gamma) \iprod{\nu_0}{V_\theta}.
\end{align*}
\end{proposition}
The proof is based on Lagrangian duality and is presented in Appendix~\ref{app:QREPS-structure}.
This proposition has several important implications. First, it shows that the optimization problem~\eqref{eq:QREPS_OP} 
can be reduced to minimizing the convex loss function $\mathcal{G}$. By analogy with the classic logistic loss, we 
will call this loss function the \emph{logistic Bellman error}, its solutions $Q_\theta$ and $V_\theta$ the 
\emph{logistic value functions}. Unlike the squared Bellman error, the logistic Bellman error is convex in the 
action-value function $Q$ its parameters $\theta$. 
Another major implication of Proposition~\ref{prop:QREPS-structure} is that it provides a simple explicit expression for 
the policy associated with $d^*$ as a function of the logistic action-value function $Q_{\theta^*}$. This is remarkable 
since no such policy parametrization is directly imposed in the primal optimization problem~\eqref{eq:QREPS_OP} as a 
constraint, but it rather emerges naturally from the overall structure we propose.

Besides convexity, the LBE has other favorable properties: when regarded as a function of $Q$, its gradient satisfies 
$\onenorm{\nabla_Q \GG(Q)} \le 2$ and is thus $2$-Lipschitz with respect to the $\ell_\infty$ norm, and it is smooth 
with parameter $\alpha + \eta$ (due to being a composition of an $\alpha$-smooth and an $\eta$-smooth function). 
These additional properties make the LBE a desirable alternative to the squared Bellman error, which is non-convex, 
non-smooth, and has unbounded gradients. Indeed, the Lipschitzness of the LBE implies that optimizing the loss via 
stochastic gradient descent does not require any gradient clipping tricks since the derivatives are bounded by default. 
In this sense, the LBE can be seen as a theoretically well-motivated alternative to the Huber loss commonly used instead 
of the squared loss for policy evaluation.

\subsection{Approximate policy iteration with \QREPS}
We now derive a more concrete algorithmic framework based on the \QREPS optimization problem. Specifically, denoting 
the set of $(\qq,d)$ pairs that satisfy the constraints of the problem~\eqref{eq:QREPS_OP} as $\MM_\Phi$, we will 
consider a mirror-descent algorithm that calculates a sequence of distributions iteratively as
\[
 (\qq_{k+1},d_{k+1}) = \argmax_{(\qq,d)\in\MM_\Phi} \iprod{\qq}{r} - \frac{1}{\eta} D(\qq\|\qqd_k) - \frac{1}{\alpha} 
H(d\|d_k).
\]
Importantly, the reference distributions in both regularization terms are chosen to be $d_k$, and $d_0$ is chosen as 
the occupancy measure induced by a fixed initial policy $\pi_0$ with full support over all actions.
By the results established above, implementing the \QREPS updates requires finding the minimum $\theta_k^*$ of the 
logistic Bellman error function 
\[
 \mathcal{G}_k(\theta) = \frac1{\eta} \log \left(\sum_{x,a}  {\qqd_k(x,a) e^{\eta \Delta_\theta(x,a)}}\right) + 
(1-\gamma) \iprod{\nu_0}{V_\theta}.
\]
We will denote the logistic value functions corresponding to $\theta_k^*$ as $Q_k^*$ and $V_k^*$, and the induced 
policy as $\pi_k^*(a|x)$.
In practice, exact minimization can be often infeasible due to 
the lack of knowledge of the transition function $P$ and limited access to computation. Thus, practical 
implementations of \QREPS will inevitably have to work with approximate minimizers $\theta_k$ of the logistic Bellman 
error $\GG_k$. We will denote the corresponding logistic value functions as $Q_k$ and $V_k$ and the policy as $\pi_k$, 
and the distribution $d_k$ will be chosen as the occupancy measure induced by $\pi_k$. 
By analogy with classical approximate policy iteration (API) schemes, we will refer to the minimization of the 
LBE~$\GG_k$ as a \emph{policy evaluation} step that is carried out by the subroutine \QREPSEVAL. Using this language, 
we present a pseudocode for \QREPS as Algorithm~\ref{alg:QREPS}.
\begin{algorithm}[t]
\SetAlgoLined
Initialize $\pi_0$ arbitrarily\;
\For{$k=1,2,\dots,K$}{
    Policy evaluation: $\theta_k = \QREPSEVAL(\pi_k)$\;
    Policy update: $\pi_{k+1}(a|x) \propto \pi_k(a|x) e^{\alpha Q_k(x,a)}$\;
 }
 \KwResult{$\pi_K$}
 \caption{\QREPS}\label{alg:QREPS}
\end{algorithm}

\subsection{Policy evaluation via saddle-point optimization}
\label{sec:eval}
In order to use \QREPS in a reinforcement-learning setting, we need to design a policy-evaluation subroutine that is 
able to directly work with sample transitions obtained through interaction with the environment. We will specifically 
consider a scheme where in each epoch $k$, we execute policy $\pi_k$ and obtain a batch of 
$N$ sample transitions $\{\xi_{k,n}\}_{n=1}^N$ , with $\xi_{k,n}=(X_{k,n},A_{k,n},X_{k,n}')$, drawn from the occupancy  
measure $\dd_k$ induced by $\pi_k$. 
Furthermore, defining the \emph{empirical Bellman error}  for any $(x,a,x')$ as
\[
\wh{\Delta}_{\theta}(x,a,x') = r(x,a) + \gamma V_\theta(x') - Q_\theta(x,a),
\]
we define the \emph{empirical logistic Bellman error} (ELBE):
\begin{equation}\label{eq:ELBE}
\begin{split}
 \wh{\dual}_k(\theta) =& \frac1\eta \log \pa{\frac 1N \sum_{n=1}^N e^{\eta \wh{\Delta}_\theta(\xi_{k,n})}}
\\
&+(1-\gamma)\iprod{\nu_0}{\Vtet}.
 \end{split}
\end{equation}
As in the case of the \REPS objective function~\eqref{eq:REPS-dual} and the squared Bellman error~\eqref{eq:sqBellman}, 
the empirical counterpart of the LBE is a biased estimator of the true loss due to the conditional expectation taken 
over $x'$ within the exponent. As we will show in Section~\ref{sec:analysis}, this bias can be directly controlled by 
the magnitude of the regularization parameter $\eta$, and convergence to the optimal policy can be guaranteed for small 
enough choices of $\eta$ corresponding to strong regularization. 

We now provide a practical algorithmic framework for optimizing the ELBE~\eqref{eq:ELBE} based on the following 
reparameterization of the loss function:
\begin{proposition}\label{prop:duality_opt}
Let $\DD_N$ be the set of all probability distributions over $[N]$ and define
\begin{align*}
\mathcal{S}_k(\theta,\ymu) =& \sum_{n} \ymu(n)\pa{\wh{\Delta}_\theta(\xi_{k,n})- \frac{1}{\eta} 
\log (N\ymu(n))} \\ 
&+(1-\gamma)\iprod{\nu_0}{\Vtet}
 \label{sp}
\end{align*}
for each $\ymu\in\DD_N$. Then, 
the problem of minimizing the ELBE can be rewritten as $\min_{\theta} \hat{\dual}_k(\theta) = \min_{\theta} 
\max_{\ymu\in\DD_N} \mathcal{S}_k(\theta,\ymu)$.
\end{proposition}
The proof is a straightforward consequence of the Donsker--Varadhan variational formula (see, e.g., \citealp{BoLuMa13}, 
Corollary 4.15). 
Motivated by the characterization above, we propose to formulate the optimization of the ELBE as a two-player game 
between a \emph{sampler} and a \emph{learner}: in each round $\tau=1,2,\dots,T$, the sampler proposes a distribution 
$z_{k,\tau} \in \DD_N$ over sample transtions and the learner updates the parameters $\theta_{k,\tau}$, together 
attempting to approximate the saddle point of $\SS_k$. In particular, the learner will update the parameters $\theta$ 
through online stochastic gradient descent on the sequence of loss functions $\ell_\tau = \SS_{k}(\cdot,z_{k,\tau})$. In 
order to estimate the gradients, we define the policy $\pi_{k,\theta}(a|x) = \pi_{k}(a|x)e^{\alpha\pa{Q_\theta(x,a) - 
V_\theta(x)}}$ and propose the following procedure: sample an index $I$ from the distribution $z_{k,\tau}$ and let
$(X,A,X')=(X_{k,I},A_{k,I},X'_{k,I})$ and sample a state $\overline{X} \sim \nu_0$ and two actions $A' \sim 
\pi_\theta(\cdot|X')$ and $\overline{A} \sim \pi_\theta(\cdot|\overline{X})$, then let
\begin{equation}\label{eq:gradest}
 \wh{g}_{k,t}(\theta) = \gamma \varphi(X',A') - \varphi(X,A) + (1-\gamma) \varphi(\overline{X},\overline{A}).
\end{equation}
This choice is justified by the following proposition:
\begin{proposition}\label{prop:lossgrad}
The vector $\wh{g}_{k,t}(\theta)$ is an unbiased estimate of the gradient $\nabla_\theta 
\SS_{k}(\theta_{k,\tau},z_{k,\tau})$.
\end{proposition}
The proof is provided in Appendix~\ref{app:lossgrad}. Using this 
gradient estimator, the learner updates $\theta_{k,\tau}$ as 
\[
 \theta_{k,\tau+1} = \theta_{k,\tau} - \beta \wh{g}_{k,t}(\theta_{k,\tau}),
\]
where $\beta>0$ is a stepsize parameter.
As for the sampler, one can consider several different algorithms for updating the distributions
$z_{k,\tau}$. A straightforward choice is simply using the best-response strategy of playing
\[
 z_{k,\tau}(n) \propto \exp\pa{\eta \wh{\Delta}_{\theta_{k,\tau}}(\xi_{k,n})},
\]
whence the overall algorithm becomes equivalent to optimizing the  empirical LBE via stochastic gradient descent. 
A slightly more sophisticated (and sometimes empirically more stable) approach is updating the parameters incrementally 
by first calculating the gradient $h_{k,\tau} = \nabla_z \SS_k(\theta_{k,\tau},z_{k,\tau})$ with components
\[
  h_{k,\tau}(n)= \wh{\Delta}_\theta(\xi_{k,n}) -\frac 1\eta \log\pa{N\ymu_{k,\tau}(n)},
\]
and then updating $z_{k,\tau}$ through an exponentiated gradient step with a stepsize $\beta'$:
\[
 z_{k,\tau+1}(n) \propto z_{k,\tau}(n) e^{\beta' h_{k,\tau}(n)}.
\]
We refer to the implementation of \QREPS using the above procedure as \OPT and provide pseudocode 
as Algorithm~\ref{alg:saddleQREPS}. We discuss the impact of the design choices involved in choosing the sampler's 
updates in Section~\ref{sec:conc}.

\begin{algorithm}[t]
\SetAlgoLined
Initialize $\pi_0$ arbitrarily\;
\For{$k=0,1,2,\dots,K-1$}{
    Run $\pi_k$ and collect sample transitions $\{\xi_{k,n}\}_{n=1}^N$\;
    Saddle-point optimization for \QREPSEVAL:\\
    \For{$\tau=1,2,\dots,T$}{
        $\theta_{k,\tau}\leftarrow \theta_{k,\tau-1}-\beta \wh{g}_{k,\tau-1}(\theta)$\;
        $z_{k,\tau}(n) \leftarrow \frac{z_{k,\tau-1}(n) \exp\pa{\beta' h_{k,\tau-1}(n)}}{\sum_m z_{k,\tau-1}(m) 
\exp\pa{\beta' h_{k,\tau-1}(m)}}$\;
    }
    $\theta_k = \frac1T \sum_{\tau=0}^{T}\theta_{k,\tau}$\;
    Policy update: $\pi_{k+1}(a|x)\propto \pi_0(a|x) e^{\alpha\sum_{i=0}^k Q_{\theta_i}(x,a)}$\;
 }
 \KwResult{$\pi_I$ with $I\sim \mbox{Unif}(K)$}
 \caption{\OPT}\label{alg:saddleQREPS}
\end{algorithm}

\section{ANALYSIS}\label{sec:analysis}
This section presents a collection of formal guarantees regarding the performance of \QREPS. For most of the analysis, 
we will make the following assumptions:
\begin{assumption}[Concentrability\footnote{Sometimes also called 
``concentratability''.}]\label{ass:conc}
 The likelihood ratio for any two valid occupancy measures $\qq$ and $\qq'$ is upper-bounded by some 
$C_\gamma$ called the \emph{concentrability coefficient}: $\sup_{x} \frac{\sum_a \qq(x,a)}{\sum_a \qq'(x,a)} \le 
C_\gamma$.
\end{assumption}
\begin{assumption}[Factored linear MDP]\label{ass:linMDP}
 There exists a function $\omega:\X\ra \real^m$ and a vector $\vartheta\in\real^m$ such that for any $x,a,x'$, the 
transition 
function factorizes as $P(x'|x,a) = \iprod{\omega(x')}{\varphi(x,a)}$ and the reward function can be expressed as 
$r(x,a) = \iprod{\vartheta}{\varphi(x,a)}$. 
\end{assumption}
The first of these ensures that every policy will explore the state space sufficiently well. Although this is a rather 
strong condition that is rarely verified in problems of practical interest, it is commonly assumed to ease theoretical 
analysis of batch RL algorithms. For instance, similar conditions are required in the classic works of 
\citet{kakade2002approximately}, \citet{szepes06learning}, and more recently by \citet{GPP17}, 
\citet{AKLM20} and \citet{XJ20}. The second assumption ensures that the feature space is expressive enough to allow the 
representation of the optimal action-value function and thus the optimal policy (a property often called 
\emph{realizability}). This condition has been first proposed by \citet{YW19} and \citet{JYWJ20}, and has quickly 
become a standard model for studying reinforcement learning algorithms under realizable linear function approximation 
\citep{CYJW20,WDYS20,NPB20,AKKS20}.

\paragraph{Error propagation of \QREPS.}
We first provide guarantees regarding the propagation of optimization errors in the general \QREPS algorithm template. 
Specifically, we will study how the suboptimality of each policy evaluation step impacts the convergence rate of the 
sequence of policies to the optimal policy in terms of the corresponding expected rewards. To this end, we let 
$\theta_k^* = \argmin_\theta \GG_k(\theta)$, and define the suboptimality gap associated with the parameter vector 
$\theta_k$ computed by \QREPSEVAL as $\varepsilon_k = \GG_k(\theta_k) - \GG_k(\theta_k^*)$. Denoting the 
normalized discounted return associated with policy $\pi_k$ as $R_k = \iprod{d}{r}$ and the optimal return $R^* = 
\iprod{d^*}{r}$, our main result is stated as follows:
\begin{theorem}\label{thm:propagation}
Suppose that Assumptions~\ref{ass:conc} and~\ref{ass:linMDP} hold and let $d^* = \argmax_{d\in\MM^*} \iprod{d}{r}$. 
Then, the policy sequence output by \QREPS satisfies
 \begin{align*}
 \sum_{k=1}^K \pa{R^* - R_k} \le& \frac{D(\qqd^*\|\qqd_{0})}{\eta} + \frac{H(d^*\|d_0)}{\alpha} + 
\sum_{k=1}^K 
\varepsilon_k
 \\
 &\,\,\,+ 3C_\gamma \pa{\sqrt{\frac{\alpha}{{1-\gamma}}} + 
\sqrt{\eta}}\sum_{k=1}^K \sqrt{\varepsilon_k}.
\end{align*}
\end{theorem}
The proof can be found in Appendix~\ref{app:mainproof}.
The theorem implies that whenever the bound increases sublinearly, the average quality of the policies output by \QREPS 
approaches that of the optimal policy: $\lim_{K\ra\infty} \frac{1}{K} \sum_{k=1}^K R_k= R^*$. 
An immediate observation is that the policy updates are perfectly accurate
(i.e., $\varepsilon_k=0$ for all $k$), then the expected return is guaranteed to converge to the optimum at a rate of 
$1/K$, as expected for mirror-descent algorithms optimizing a fixed linear loss, and also matching the best known rates 
for natural policy gradient methods \citep{AKLM20}.
In the more interesting case where the evaluation steps are not perfect, the correct choice of the regularization 
parameters depends on the magnitude of the evaluation errors. Theorem~\ref{thm:conc} below provides  bounds on these 
errors when using the minimizer of the empirical LBE for policy evaluation.

\begin{figure*}[t]
\includegraphics[width=\textwidth]{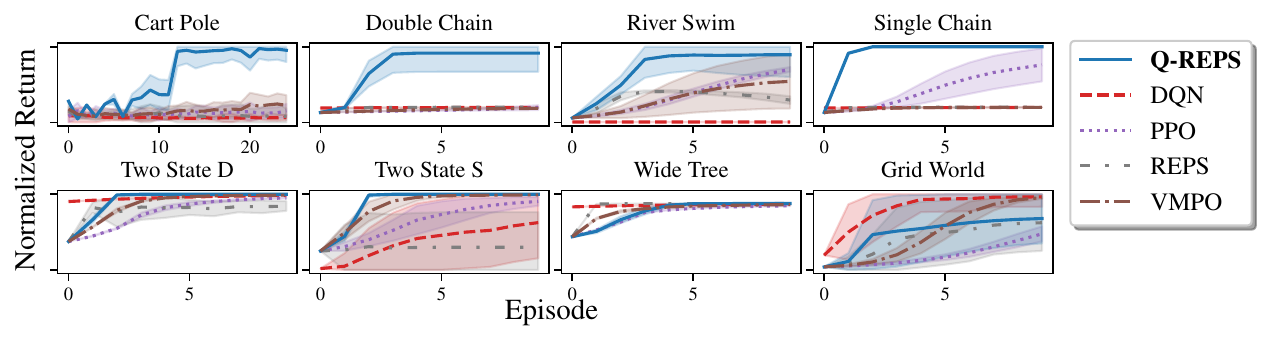}
\caption{Empirical performance \QREPS on different benchmarks. The returns are scaled to $[0, 1]$ by dividing by the 
maximum achievable return, with the mean plotted in solid lines and and the shaded area representing one standard 
deviation.}
\label{fig:environments}
\end{figure*}

One important feature of the bound of Theorem~\ref{thm:propagation} is that it shows no direct dependence on the size 
of the MDP or the dimensionality of the feature map, which can be seen to justify using \QREPS with
general (possibly non-linear) function approximation. Indeed, observe that every MDP can be seen to satisfy 
Assumption~\ref{ass:linMDP} when choosing $\Phi$ as the identity map, and that the logistic Bellman error can be 
directly written as a function of the Q-functions. Then, Theorem~\ref{thm:propagation} shows that whenever one is able 
to keep the policy evaluation errors small, convergence to the optimal policy can be guaranteed irrespective of the size 
of the state space. 
Among other implications, this suggests that the logistic Bellman error can indeed be a viable objective function for 
large-scale deep reinforcement learning. 

\paragraph{Concentration of the empirical LBE.}
We now move on to establishing some important properties of the empirical logistic Bellman error~\eqref{eq:ELBE}. 
For simplicity, we will assume that the sample transitions are generated in an i.i.d.~fashion: each 
$(X_{k,n},A_{k,n})$ is drawn independently from $\qq_k$ and $X_{k,n}'$ is drawn independently from 
$P(\cdot|X_{k,n},A_{k,n})$. Under this condition, the following theorem establishes the connection between the ELBE and 
the true LBE:
\begin{theorem}\label{thm:conc}
Let $\QQ = \ev{Q_\theta: \infnorm{Q_\theta} \le B'}$ for some $B'>0$ and $\Theta$ be the corresponding set of parameter 
vectors. Furthermore, define $B = 1+(1+\gamma)B'$, and assume that $\eta B \le 1$ holds. Then, with probability at 
least $1-\delta$, the following holds:
\[
 \sup_{\theta\in\Theta} \left|\wh\GG_k(\theta) - \GG_k(\theta) \right| \le 8 \eta B^2 + 56 
\sqrt{\frac{m\log\bpa{(1+4BN)/\delta}}{N}}.
\]
\end{theorem}
In Appendix~\ref{app:conc}, we provide a more detailed statement of the theorem that holds for general 
Q-function classes, as well as the proof. The main feature of this theorem is quantifying the bias of the empirical 
LBE, showing that it is proportional to the regularization parameter $\eta$, making it possible to tune the 
parameters of \QREPS in a way that ensures convergence to the optimal policy.

\paragraph{\QREPS performance guarantees.}
Putting the results from the previous sections together, we obtain 
the following performance guarantee for \QREPS:
\begin{corollary}
Suppose that Assumptions~\ref{ass:conc} and~\ref{ass:linMDP} hold and that each update \QREPS is implemented by 
minimizing the empirical LBE~\eqref{eq:ELBE} evaluated on $N$ independent sample transitions. Furthermore, suppose 
that $\infnorm{Q_k^*} \le B$ for all $k$.
Then, setting $N=K$ and tuning $\eta$ and $\alpha$ appropriately,
\QREPS is guaranteed to output an $\epsilon$-optimal policy with 
\[
\epsilon = \wt{O}\pa{\pa{\sqrt{\frac{1}{1-\gamma}}\!+\!B} \sqrt{\frac{m C_\gamma D(\dd^*\|\dd_0)}{K}}}.
\]
Furthermore, for any $\epsilon>0$ and the same choice of $\eta$ and $\alpha$, \QREPS is guaranteed 
to output an $\epsilon$-optimal policy after observing $T_\epsilon = NK$ transitions with
\[
 T_\epsilon = \wt{O}\pa{\frac{\pa{\pa{B^2 + \frac{1}{1-\gamma}} m C_\gamma D(\dd^*\|\dd_0)}^2}{\epsilon^4}}.
\]
\end{corollary}

\section{EXPERIMENTS}

In this section we evaluate \QREPS empirically. 
As the algorithm is essentially on-policy, we compare it with: DQN using Polyak averaging and getting new samples at  
every episode \citep{mnih2015human}; PPO as a surrogate of TRPO \citep{schulman2017proximal}; VMPO as the 
on-policy version of MPO \citep{song2019v}; and REPS with parametric policies \citep{deisenroth2013survey}. 
The code used for these experiments is available online at \url{https://github.com/sebascuri/qreps}.

We evaluate these algorithms in different standard environments which we describe in Appendix \ref{app:extra_experiments}.
In all environments we use indicator features, except for Cart-Pole that we use 
the initialization of a 2-layer ReLU Neural Network as features and optimize the last layer.
For all environments but CartPole we run episodes of length 200 and update the policy at the end of each episode. 
Due to the early-termination of CartPole, we run episodes until termination or length 200 and update the policy after 4 episodes.

In Figure~\ref{fig:environments} we plot the sample mean and one standard deviation of 50 independent runs of the 
algorithms (random seeds 0 to 49). 
In all cases, \QREPS either outperforms the competing algorithms or is at least comparable them.

\section{CONCLUSION}\label{sec:conc}
Due to its many favorable properties, we believe that \QREPS has significant potential to become a state-of-the-art 
method for reinforcement learning. That said, there is still a lot of room for improvement on both fronts of 
theoretical guarantees and practical applicability. We outline some challenges for future research and 
discuss some implications of our results below.

\vspace{-.1cm}

\paragraph{Limitations of our theory.} 
While our theoretical guarantees have several desirable properties, they also have a number of shortcomings. 
First, while the error-propagation guarantee of Theorem~\ref{thm:propagation} has no explicit dependence 
on the number of states, it requires a very restrictive concentrability condition to hold. We believe that this is an 
artifact of our analysis and expect that it can be removed by a more careful proof technique. Similarly, our 
Theorem~\ref{thm:conc} shows that the bias of the straightforward empirical estimator of the LBE can be  
controlled by the regularization parameter $\eta$, but it comes with the caveat that it requires the condition that the 
logistic Q-functions be bounded. While we were not able to prove an explicit upper bound on the Q-functions, our 
extensive supplementary experiments indicate that they are bounded by a constant independent of $\eta$, and we believe 
that a more sophisticated analysis could formally establish this property. In light of these limitations, we prefer 
to think of the guarantees of Theorems~\ref{thm:propagation} and~\ref{thm:conc} as promising initial results, and 
we leave the important challenge of tightening these guarantees open for future work.

\vspace{-.1cm}

\paragraph{Limitations of our algorithm.}
The most important merit of \QREPS is that it can be implemented without any significant deviation from its theoretical 
specifications. The most serious implementation issue is that \QREPS requires sampling from the discounted occupancy 
measure, which can only be done efficiently when having access to a reset action. This is a common issue of many 
reinforcement learning algorithms that is often addressed by using samples from the undiscounted 
state-action distribution. This heuristic often leads to well-performing practical algorithms, but has been long known 
to suffer from bias issues, as pointed out by \citet{Tho14} and \citet{NT20}. We expect that this heuristic could help 
practical implementations of \QREPS, although it should be applied with caution.
Another practical limitation of our algorithm is that it requires storing the cumulative sum 
of all past Q-functions, which is not feasible without approximations in a deep RL implementation. 
It is straightforward to address this limitation by adjusting the regularization terms, but it is currently unclear 
if it is still possible to meaningfully control the error propagation of the resulting variant.

\paragraph{Experience replay and \OPT.} Interestingly, the saddle-point optimization scheme proposed in 
Section~\ref{sec:eval} can be seen as a principled form of \emph{prioritized experience replay} where the samples 
used for value-function updates are drawn according to some priority criteria \citep{schaul2015prioritized}. Indeed, 
this method maintains a probability distribution over sample transitions that governs the value updates, with the 
distribution being adjusted after each update according to a rule that is determined by the TD error. Different rules 
for the priority updates result in different learning dynamics with the best choice potentially depending on the problem 
instance. In our experiments, we have observed that best-response updates are sometimes overly aggressive, and the 
incremental updates featured in Algorithm~\ref{alg:saddleQREPS} lead to more stable behavior. 
We leave a formal study of the best practices and uncovering further connections with prioritized experience replay 
for future research.

\vspace{-.1cm}

\paragraph{The relaxed LP formulation.} Our method is based on a subtle variation on the classic LP formulation of 
optimal control in MDPs due to \citet{Man60}. One key element in our formulation is a linear relaxation of some of the 
constraints in this LP, which is a technique looking back to a long history: a similar relaxation has been first 
proposed by \citet{schwesei85}, whose approach was later popularized 
by the influential work of \citet{FR03}. This latter paper initiated a long line of work studying the properties of 
solutions to various linearly relaxed versions of the LP, mostly focusing on the quality of value functions extracted 
from the solutions (see, e.g., \citealp{PZ09,DFM12,LBS17}). Another complementary line of work was initiated by 
\citet{peters10reps}, whose main goal was deriving practical RL algorithms from a relaxed LP formulation. Our own work 
is heavily influenced by this latter line of research, in that our main focus is also on algorithmic aspects.
That said, one important result in our paper is providing a sufficient condition for the LP relaxation to 
yield exact solutions to the original LP: our analysis shows that for factored linear MDPs, the relaxation we 
propose suffers from no approximation error (cf.~Proposition~\ref{prop:real}). Understanding the approximation errors 
without this structural assumption is a very exciting question that we plan to address in future work, building on the 
approximate linear programming literature initiated by \citet{FR03}. 
Similarly, we expect that our algorithmic techniques can be combined with other, more sophisticated relaxation methods. 
In light of this discussion, we view our work as a promising step toward bridging 
the gap between LP-based approximate dynamic-programming approaches and mainstream reinforcement learning.

\subsubsection*{Acknowledgements}
This project has received funding from the European Research Council (ERC) under the European Unions Horizon 2020 
research and innovation program grant agreement No 815943. G.~Neu was supported by ``la Caixa'' Banking Foundation 
through the Junior Leader Postdoctoral Fellowship Programme, a Google Faculty Research Award, and the Bosch AI Young 
Researcher Award.

\bibliographystyle{apalike}
\bibliography{ngbib,shortconfs}

\appendix
\onecolumn
\section{Omitted proofs}
\subsection{The proof of Proposition ~\ref{prop:QREPS-structure}}\label{app:QREPS-structure}
The proof is based on Lagrangian duality: we introduce a set of multipliers $V\in\real^X$ and $\theta\in\real^m$ for 
the two sets of constraints and $\rho$ for the normalization constraint of $\qq$, and write the Lagrangian of the 
constrained optimization problem~\eqref{eq:QREPS_OP} as
\begin{align}
 \LL(\qq,\dd;V,\theta,\rho) =& \iprod{\qq}{r} + \iprod{V}{\gamma P\transpose \qq  + (1-\gamma)\nu_0 - E\transpose \dd} 
+ 
\iprod{\theta}{\Phi\transpose \dd - \Phi\transpose \qq} + \rho \pa{1 - \iprod{\qq}{\bone}} - \frac{1}{\eta} 
D(\qq\|\qq_0)- 
\frac{1}{\alpha} H(\dd\|\dd_0)\nonumber
\\
 =& \iprod{\qq}{r + \gamma P V - \Phi\theta - \rho \bone} + \iprod{\dd}{\Phi\theta - E V} + (1-\gamma) 
\iprod{\nu_0}{V} + \rho - \frac{1}{\eta} D(\qq\|\qq_0)- \frac{1}{\alpha} H(\dd\|\dd_0)\nonumber
\\
 =& \iprod{\qq}{\Delta_{\theta,V} - \rho \bone} + \iprod{\dd}{Q_\theta - E V} + (1-\gamma) 
\iprod{\nu_0}{V} + \rho - \frac{1}{\eta} D(\qq\|\qq_0)- \frac{1}{\alpha} H(\dd\|\dd_0),\label{eq:QREPS_Lagrangian}
\end{align}
where we used the notation $Q_\theta = \Phi\theta$ and $\Delta_{\theta,V} = r + \gamma PV - Q_\theta$ in the last line. 
Notice that the above is a strictly concave function of $d$ and $\qq$, so its maximum can be found by setting the 
derivatives with respect to these parameters to zero. In order to do this, we note that
\[
 \frac{\partial D(\qq\|\qq_0)}{\partial \qq(x,a)} = \log \qq(x,a) - \log \qq_0(x,a) \qquad \mbox{and } \qquad 
\frac{\partial 
H(\dd\|\dd_0)}{\partial \dd(x,a)} = \log \pi_\dd(a|x) - \log \pi_{0}(a|x),
\]
where $\pi_\dd(a|x) = \dd(x,a) / \sum_{a'} \dd(x,a')$ and the last expression can be derived by straightforward 
calculations 
(see, e.g., Appendix~A.4 in \citealp{NJG17}). This gives the following expressions for the optimal choices of $\qq$ and 
$\dd$:
\[
 \qq^*(x,a) = \qq_0(x,a) e^{\eta (\Delta_{\theta,V}(x,a)-\rho)} \qquad\mbox{and}\qquad \qquad \pi_\dd^*(x,a) 
= \pi_0(a|x) 
e^{\alpha (Q_\theta (x,a)-V(x))}.
\]
From the constraint $\sum_{x,a}\qq^*(x,a)=1$, we can express the optimal choice of $\rho$ as
\[
\rho^*=\log \left(\sum_{x,a}  {\qq_0(x,a) e^{\eta \Delta_{\theta,V}(x,a)}}\right).
\]
Similarly, from the constraint $\sum_a \pi_\dd^*(a|x)=1$, we can express  $V$ as a function of 
$\theta$ for all $x$:
\begin{align*}
V_\theta(x)=\frac1{\alpha}\log\left(\sum_a  \pi_0(a|x) e^{\alpha Q_\theta(x,a)}\right)
\end{align*}
This implies that $\dd^*$ has the form $\dd^*(x,a) = \nu(x) \pi_\dd^*(a|x)$, where $\nu$ is some nonnegative 
function 
on 
the state space. Recalling the definition of $\Delta_\theta = r + \gamma P V_\theta - Q_\theta$ and plugging the above 
parameters $(\qq^*,\dd^*,\rho^*,V_\theta)$ back into the Lagrangian~\eqref{eq:QREPS_Lagrangian} gives 
\begin{align*}
 \mathcal{G}(\theta) =& \LL(\qq^*,\dd^*;V_\theta, \theta, \rho^*) \\
 =& \sum_{x,a} \pa{\qq_0(x,a) e^{\eta (\Delta_\theta(x,a)-\rho^*)}  (\Delta_\theta(x,a)-\rho^*)+ \nu(x) \pi_0(a|x) 
e^{\alpha \left(  Q_\theta(x,a)- V_\theta(x) \right)} \left(  Q_\theta(x,a)- V_\theta(x) \right)} \\
 &- \sum_{x,a}\frac{1}{\eta} \pa{\qq_0(x,a) e^{\eta (\Delta_\theta(x,a)-\rho^*)}\log \frac{\qq_0(x,a) e^{\eta 
(\Delta_\theta(x,a)-\rho^*)}}{\qq_0(x,a)} + \qq_0(x,a)-\qq_0(x,a)e^{\eta (\Delta_\theta(x,a)-\rho^*)}}\\
 &- \sum_{x,a}\frac{1}{\alpha} \nu(x) \pi_0(a|x) e^{\alpha \left(  Q_\theta(x,a)- V_\theta(x) \right)}\log 
\frac{\pi_{0}(a|x) e^{\alpha \left(  Q_\theta(x,a)- V_\theta(x) \right)}}{\pi_{0}(a|x)} 
+(1-\gamma)\iprod{\nu_0}{V} + \rho^*\\ 
 =& \sum_{x,a} \pa{\qq_0(x,a) e^{\eta (\Delta_\theta(x,a)-\rho^*)}  (\Delta_\theta(x,a)-\rho^*)+ \nu(x) \pi_0(a|x) 
e^{\alpha \left(  Q_\theta(x,a)- V_\theta(x) \right)} \left(  Q_\theta(x,a)- V_\theta(x) \right)} \\
 &- \sum_{x,a} \qq_0(x,a) e^{\eta (\Delta_\theta(x,a)-\rho^*)} (\Delta_\theta(x,a)-\rho^*)\\
 &- \sum_{x,a} \nu(x) \pi_0(a|x) e^{\alpha \left(  Q_\theta(x,a)- V_\theta(x) \right)} \left(  Q_\theta(x,a)- 
V_\theta(x) \right) +(1-\gamma)\iprod{\nu_0}{V} + \rho^*\\
 =&(1-\gamma)\iprod{\nu_0}{V}+\rho^* \\
 =& (1-\gamma) \iprod{\nu_0}{V} + \frac1{\eta} \log \left(\sum_{x,a}  {\qq_0(x,a) e^{\eta \Delta_\theta(x,a)}}\right).
\end{align*}
Furthermore, observe that since the parameters were chosen so that all constraints are satisfied, we also have
\begin{equation}\label{eq:Gdual}
 \mathcal{G}(\theta) = \LL(\qq^*,\dd^*;V_\theta, \theta, \rho^*) = \iprod{\qq^*}{r} - \frac{1}{\eta} 
D(\qq^*\|\qq_0)- \frac{1}{\alpha} H(\dd^*\|\dd_0).
\end{equation}
Thus, the solution of the optimization problem~\eqref{eq:QREPS_OP} can be indeed written as
\[
 \max_{\qq,\dd \ge 0} \min_{\theta,V,\rho} \LL(\qq,\dd;V,\theta,\rho) = \min_{\theta,V,\rho} \max_{\qq,\dd \ge 0} 
\LL(\qq,\dd;V,\theta,\rho) = \min_\theta \LL(\qq^*,\dd^*;V_\theta, \theta, \rho^*) = \min_\theta \mathcal{G}(\theta),
\]
which concludes the proof.
\qed

\subsection{The proof of Theorem~\ref{thm:propagation}}\label{app:mainproof}
The proof of this result is somewhat lengthy and is broken down into a sequence of lemmas and propositions. 

Before analyizing the algorithm, we first establish an important realizability property of factored linear MDPs. 
Precisely, this result will show that, under Assumption~\ref{ass:linMDP}, the relaxed constraint set $\MM_{\Phi}$ 
matches the set of valid discounted occupancy measures in an appropriate sense, which can be seen as the bare minimum 
requirement for being able to show convergence to the optimal policy.  We refer to this condition as \emph{primal 
realizability}, and show that it holds in the following sense:
\begin{proposition}\label{prop:real}
Let $\MM_{\Phi}' = \ev{\dd: (\qq,\dd)\in\MM_{\Phi}}$. Then, under Assumption~\ref{ass:linMDP}, $\MM^*=\MM'_{\Phi}$ 
holds. Furthermore, letting 
$(\qq^*,\dd^*) = \argmax_{(\qq,\dd)\in\MM_{\Phi}} \iprod{\qq}{r}$, we have $\iprod{\dd^*}{r} = \max_{\qq\in\MM^*} 
\iprod{\qq}{r}$.
\end{proposition}
\begin{proof}
 It is easy to see that $\MM^*\subseteq\MM_{\Phi}'$: for any $\qq\in\MM^*$, we can choose $\dd=\qq$ and directly verify 
that all constraints of~\eqref{eq:QREPS_OP} are satisfied. For proving the other direction, it is helpful to define the 
operator $M$ through its action $Mv = \sum_{x} \omega(x) v(x)$ for any $v\in\real^\X$, so that the condition of 
Assumption~\ref{ass:linMDP} can be expressed as $P = \Phi M$ and $r = \Phi \vartheta$.
Then, for any $(\qq,\dd)\in\MM_{\Phi}$, we write
\begin{align*}
 E\transpose \dd &= \gamma P\transpose \qq + (1-\gamma) \nu_0 = \gamma M\transpose \Phi\transpose \qq + (1-\gamma) 
\nu_0 
 \\
 &= \gamma M\transpose \Phi\transpose \dd + (1-\gamma) \nu_0 = \gamma P\transpose \dd + (1-\gamma) \nu_0.
\end{align*}
Combined with the fact that $\dd$ is non-negative, this implies that $\dd\in\MM^*$ and thus that 
$\MM_{\Phi}'\subseteq \MM^*$. Together with the previous argument, this shows that $\MM^* = \MM_{\Phi}'$ indeed holds. 
For proving the second statement, we use the assumption on 
$r$ to write $\iprod{\qq}{r} = \iprod{\Phi\transpose \qq}{\Phi r} = \iprod{\Phi\transpose \dd}{\Phi r} = 
\iprod{\dd}{r}$ for any feasible $(\qq,\dd)$. Using this fact for the maximizer $\dd^*$ implies 
$\iprod{\dd^*}{r} = \max_{\dd\in\MM'_{\Phi}} \iprod{\dd}{r} = \max_{\qq\in\MM^*} \iprod{\qq}{r}$, which concludes the 
proof.
\end{proof}

We can now turn to the analysis of \QREPS. We first introduce some useful notation and outline the main challenges 
faced 
in the proof. We start by defining the action-value functions $Q_k = \Phi\theta_k$ and $Q_k^* = \Phi \theta_k^*$, the 
state-action distributions 
\begin{align*}
 \tq_{k}(x,a) = \qq_{k-1}(x,a) e^{\eta \pa{\Delta_{\theta_k}(x,a) - \rho_k}} \qquad\mbox{and} \qquad 
\qq_{k}^*(x,a) = 
\qq_{k-1}(x,a) e^{\eta \pa{\Delta_{\theta_k^*}(x,a) - \rho_k^*}},
\end{align*}
for appropriately defined normalization constants $\rho_k$ and $\rho_k^*$ and the policies 
\begin{align*}
 \pi_{k}(a|x) = \pi_{k-1}(a|x) e^{\alpha \pa{Q_{\theta_k}(x,a) - V_{\theta_k}(x)}} \qquad\mbox{and} \qquad 
\pi_{k}^*(a|x) = \pi_{k-1}(a|x) e^{\alpha \pa{Q_{\theta_k^*}(x,a) - V_{\theta_k^*}(x)}}.
\end{align*}
A crucial challenge we have to 
address in the analysis is that, since $\theta_k$ is not the exact minimizer of $\GG_k$, the state-action distribution 
$\tq_k$ is not a valid occupancy measure. In order to prove meaningful guarantees about the performance of the 
algorithm, we need to consider the actual occupancy measure induced by policy $\pi_k$. We denote this occupancy measure 
as $\qqd_k$ and define it for all $x,a$ as
\[
 \qqd_k(x,a) = (1-\gamma)\EEs{\sum_{t=0}^\infty \gamma^t \II{(x_t,a_t)=(x,a)}}{\pi_k},
\]
where the notation emphasizes that the actions are generated by policy $\pi_k$. A major part of the proof is 
dedicated to accounting for the discrepancy between $\qq_k$ and the ideal updates $\qq_k^*$. During the proof, we will 
often factorize occupancy measures as $\qqd(x,a) = \nu(x)\pi(a|x)$, where $\nu$ is the discounted state-occupancy 
measure induced by $\pi$. In particular, we will use the notations
\begin{align*}
 \dd_{k}(x,a) = \nu_k(x) \pi_k(a|x) \qquad\mbox{and} \qquad \dd_{k}^*(x,a) = \nu_k^*(x) \pi_k^*(a|x),
\end{align*}
to refer to the state-action occupancy measures respectively induced by $\pi_k$ and $\pi^*_k$.

Our first lemma presents an important technical result that relates the suboptimality gap $\varepsilon_k$ to the 
divergence between the ideal and realized updates. 
\begin{lemma}\label{lem:projection}
$\varepsilon_k = \frac{D(\qq_k^*\|\tq_k)}{\eta} + \frac{H(\dd_k^*\|\dd_k)}{\alpha}$. 
\end{lemma}
Notably, this result does not require any of Assumptions~\ref{ass:conc} or~\ref{ass:linMDP}, as its proof only uses 
the 
properties of the optimization problem~\eqref{eq:QREPS_OP}.
\begin{proof}
The proof uses the feasibility of $(\qq_k^*,\dd_k^*)$ that follows from their definition.
We start by observing that 
\begin{align*}
 D(\qq_k^*\|\tq_k) &= \sum_{x,a} \qq_k^*(x,a) \log \frac{\qq_k^*(x,a)}{\tq_k(x,a)} 
 \\
 &= \eta \biprod{\qq_k^*}{r + \gamma P V_k^* - Q_k^* - \rho_k^*\bone - r - \gamma P V_k + Q_k + \rho_k\bone}
 \\
 &= \eta \biprod{\dd_k^*}{E V_k^* - E V_k} + \eta \iprod{\Phi\transpose \qq_k^*}{\theta_k  - \theta_k^*} + \eta (\rho_k 
+ 
(1-\gamma) \iprod{\nu_0}{V_k} - \rho_k^* - (1-\gamma) \iprod{\nu_0}{V_k^*}) 
\\
&\qquad\qquad\qquad\mbox{(using $\dd_k^* = \gamma P\transpose \qq_k^* + (1-\gamma) \nu_0$ and $Q_k - Q_k^* = 
\Phi(\theta_k - \theta^*_k)$)}
 \\
 &= \eta \biprod{\dd_k^*}{E V_k^* - E V_k} + \eta \iprod{\Phi\transpose \dd_k^*}{\theta_k  - \theta_k^*} 
 + \eta (\GG_k(\theta_k) - \GG_k(\theta_k^*)) 
 \\
&\qquad\qquad\qquad\mbox{(using $\Phi\transpose \dd_k^* = \Phi\transpose \qq_k^*$ and the form of $\GG_k$)}
 \\
 &= \eta \biprod{\dd_k^*}{E V_k^* - Q_k^* - E V_k + Q_k} + \eta (\GG_k(\theta_k) - \GG_k(\theta_k^*)).
\end{align*}
On the other hand, we have
\begin{align*}
 H(\dd_k^*\|\dd_k) &= \sum_{x,a} \dd_k^*(x,a) \log \frac{\pi_k^*(a|x)}{\pi_k(a|x)} 
 = \alpha \iprod{\dd_k^*}{Q_k^* - EV_k^* - Q_k + EV_k}.
\end{align*}
Putting the two equalities together, we get
\[
 \frac{D(\qq_k^*\|\tq_k)}{\eta} + \frac{H(\dd_k^*\|\dd_k)}{\alpha} = \GG_k(V_k) - \GG_k(V_k^*)
\]
as required.  
\end{proof}
The next result shows that, as a consequence of the above property, the realized occupancy measure $d_k$ will be close 
to the ideal one, $d_k^*$. The proof only uses Assumption~\ref{ass:linMDP} to make sure that $\dd_k^*$ is a 
valid occupancy 
measure.
\begin{lemma}\label{lem:qddist}
Suppose that Assumption~\ref{ass:linMDP} holds. Then, $D\pa{\dd_k^*\middle\|\qqd_k} \le 
\frac{H\pa{\dd_k^*\middle\|\qqd_k}}{1-\gamma}.$
\end{lemma}
\begin{proof}
The proof follows from direct calculations and exploiting several properties of the relative entropy:
\begin{align*}
 D\pa{\dd_k^*\middle\|\qqd_k} &= D\pa{\nu_k^*\middle\|\nu_k} + H\pa{\dd_k^*\middle\|\qqd_k}
 \\
 &\qquad\qquad\mbox{(by the chain rule of the relative entropy)}
 \\
 &= D\pa{(1-\gamma) \nu_0 + \gamma P\transpose \dd_k^* \middle\|(1-\gamma) \nu_0 + \gamma P\transpose \qq_k} 
 + H\pa{\dd_k^*\middle\|\qqd_k}
\\
 &\qquad\qquad\mbox{(using that $\dd_k^*$ and $\qqd_k$ are valid occupancy measures)}
\\
&\le (1-\gamma) D\pa{\nu_0 \middle\|\nu_0} + \gamma D\pa{P\transpose \dd_k^* \middle\|P\transpose \qqd_k}
+ H\pa{\dd_k^*\middle\|\qqd_k}
\\
 &\qquad\qquad\mbox{(using the joint convexity of the relative entropy)}
\\
&\le \gamma D\pa{\dd_k^* \middle\|\qq_k} + H\pa{\dd_k^*\middle\|\qqd_k},
\end{align*}
where the final step follows from the using information-processing inequality for the relative entropy. Reordering the 
terms concludes the proof.
\end{proof}
Armed with the above two lemmas, we are now ready to present the proof of Theorem~\ref{thm:propagation}.
\begin{proof}[Proof of Theorem~\ref{thm:propagation}]
The proof is based on direct calculations inspired by the classical mirror descent analysis. We let $d^* = 
\argmax_{d\in\MM^*}\iprod{d}{r}$ and $\qq^*$ be any state-action distribution satisfying $(\qq^*,d^*)\in\MM_\Phi$.
We first express the 
divergence between the comparator $\qq^*$ and the unprojected iterate $\tq_k$:
\begin{align*}
 D(\qq^*\|\tq_k) &= \sum_{x,a} \qq^*(x,a) \log \frac{\qq^*(x,a)}{\tq_k(x,a)} 
 = \sum_{x,a} \qq^*(x,a) \log \frac{\qq(x,a)}{\qqd_{k-1}(x,a)} - \sum_{x,a} \qq^*(x,a) \log 
\frac{\tq_{k}(x,a)}{\qqd_{k-1}(x,a)}
\\
& = D(\qq^*\|\qqd_{k-1}) - \eta \iprod{\qq^*}{r + \gamma PV_k - Q_k} + \eta \rho_k
\\
& = D(\qq^*\|\qqd_{k-1}) - \eta \iprod{\qq^*}{r - \Phi\theta_k} + \eta \iprod{\dd^*}{EV_k} + \eta \bpa{\rho_k + 
(1-\gamma)\iprod{\nu_0}{V_k}}
\\
&\qquad\qquad\qquad\mbox{(using $\dd^* = \gamma P\transpose \qq^* + (1-\gamma) \nu_0$ and $Q_k = 
\Phi\theta_k$)}
\\
& = D(\qq^*\|\qqd_{k-1}) - \eta \iprod{\qq^*}{r} + \eta\iprod{\dd^*}{E V_k - \Phi\theta_k} + \eta \GG_k(\theta_k)
\\
&\qquad\qquad\qquad\mbox{(using $\Phi\transpose \dd^* = \Phi\transpose \qq^*$ and the form of $\GG_k$)}
\\
& \le D(\qq^*\|\qqd_{k-1}) - \eta \iprod{\qq^*}{r} + \eta\iprod{\dd^*}{E V_k - \Phi\theta_k} + \eta \GG_k(\theta_k^*) + 
\eta 
\varepsilon_k
\\
&\qquad\qquad\qquad\mbox{(using the suboptimality guarantee of $\theta_k$)}
\\
& \le D(\qq^*\|\qqd_{k-1}) - \eta \iprod{\qq^*}{r} + \eta\iprod{\dd^*}{E V_k - \Phi\theta_k} + \eta \iprod{\qq_k^*}{r} 
- 
D(\qq_k^*\|\qqd_{k-1}) - \frac{\eta H(\dd^*_k\|\dd_{k-1})}{\alpha}
+ \eta \varepsilon_k
\\
&\qquad\qquad\qquad\mbox{(using the dual form~\eqref{eq:Gdual} of $\GG_k(\theta_k)$)}
\\
& \le D(\qq^*\|\qqd_{k-1}) - \eta \iprod{\qq^*}{r} + \eta\iprod{\dd^*}{E V_k - \Phi\theta_k} + \eta \iprod{\dd_k}{r} + 
\eta \iprod{\dd_k^*-\dd_k}{r} + \eta \varepsilon_k
\\
&\qquad\qquad\qquad\mbox{(using that $\iprod{\dd_k^*}{r} = \iprod{\qq_k^*}{r}$ by 
Proposition~\ref{prop:real})}
\\
& \le D(\qq^*\|\qqd_{k-1}) - \eta \iprod{\qq^*}{r} + \eta\iprod{\dd^*}{E V_k - \Phi\theta_k} + \eta \iprod{\dd_k}{r} + 
\eta 
\onenorm{\dd_k^*-\dd_k} + \eta \varepsilon_k,
\end{align*}
where we used $\infnorm{r} \le 1$ in the last step.
After reordering and noticing that $\iprod{\qq^*}{r} = \iprod{\dd^*}{r}$, we obtain
\[
 \iprod{\dd^* - \dd_k}{r} \le \frac{D(\qq^*\|\qqd_{k-1}) - D(\qq^*\|\tq_{k})}{\eta} + \iprod{\dd^*}{E V_k - Q_k}
 + \eta \onenorm{\dd_k^*-\dd_k} + \varepsilon_k.
\]
Furthermore, we have
\begin{align*}
 H(\dd^*\|\dd_k) &= \sum_{x,a} \dd^*(x,a) \log \frac{\pi^*(a|x)}{\pi_k(a|x)} = 
 \sum_{x,a} \dd^*(x,a) \log \frac{\pi^*(a|x)}{\pi_{k-1}(a|x)} - \sum_{x,a} \qq(x,a) \log 
\frac{\pi_k(a|x)}{\pi_{k-1}(a|x)}
 \\
 &= H(\dd^*\|\dd_{k-1}) - \alpha \iprod{\dd^*}{Q_k - EV_k}.
\end{align*}
Plugging this equality back into the previous bound, we finally obtain
\begin{align*}
 &\iprod{\dd^* - \dd_k}{r} \le 
 \frac{D(\qq^*\|\qqd_{k-1}) - D(\qq^*\|\tq_{k})}{\eta} + \frac{H(\dd^*\|\dd_{k-1}) - 
H(\dd^*\|\dd_k)}{\alpha} + \onenorm{\dd_k^*-\dd_k} + \varepsilon_k
\\
&\qquad = \frac{D(\qq^*\|\qqd_{k}) - D(\qq^*\|\tq_{k})}{\eta} + \frac{D(\qq^*\|\qqd_{k-1}) - 
D(\qq^*\|\qqd_{k})}{\eta} + \frac{H(\dd^*\|\dd_{k-1}) - 
H(\dd^*\|\dd_k)}{\alpha} + \onenorm{\dd_k^*-\dd_k} + \varepsilon_k.
\end{align*}
Summing up for all $k$ and omitting some nonpositive terms, we obtain
\begin{equation}\label{eq:regbound}
 \sum_{k=1}^K \iprod{\dd^* - \dd_k}{r} \le \frac{D(\qq^*\|\qqd_{0})}{\eta} + \frac{H(\dd^*\|\dd_0)}{\alpha}
 + \sum_{k=1}^K \pa{\frac{D(\qq^*\|\qqd_{k}) - D(\qq^*\|\tq_{k})}{\eta} + \onenorm{\dd_k^*-\dd_k} + \varepsilon_k}
\end{equation}
Combining Lemma~\ref{lem:qddist} with Pinsker's inequality, we can bound 
\[
\onenorm{\dd_k^*-\dd_k} \le \sqrt{2 D(\dd_k^*\|\dd_k)} \le \sqrt{\frac{2 H(\dd_k^*\|\dd_k)}{1-\gamma}} \le 
\sqrt{\frac{2 \alpha \varepsilon_k}{1-\gamma}},
\]
where in the last step we also used Lemma~\ref{lem:projection} that implies $H(\dd_k^*\|\td_k) \le \alpha 
\varepsilon_k$. Thus, the remaining challenge is to bound the terms 
$D(\qq^*\|\qqd_{k}) - D(\qq^*\|\tq_{k})$. In order to do this, let us introduce the Bregman projection of $\tnu_k$ to 
the space of occupancy measures, $\tnu_k^* = \argmin_{\nu\in\Delta_\gamma(\X)} D(\nu\|\tnu_k)$. Then, we can write
\begin{align*}
 D(\qq^*\|\qqd_k) - D(\qq^*\|\tq_k) =& D(\nu^*\|\nu_k) - D(\nu^*\|\tnu_k) 
 \\
 =& D(\nu^*\|\nu_k) - D(\nu^*\|\tnu_k^*) + 
D(\nu^*\|\tnu_k^*) -  D(\nu^*\|\tnu_k)
\\
\le& D(\nu^*\|\nu_k) - D(\nu^*\|\tnu_k^*) - D(\tnu_k^*\|\tnu_k) \le D(\nu^*\|\nu_k) - D(\nu^*\|\tnu_k^*) ,
\end{align*}
where the first inequality is the generalized Pythagorean inequality that uses the fact that $\tnu_k^*$ is the 
Bregman 
projection of $\tnu_k$ (cf.~Lemma~11.3 in \citealp{CBLu06:book}). By using the chain rule of the relative entropy and 
appealing to Lemma~\ref{lem:qddist}, we have
\[
 D(\nu_k^*\|\tnu_k) = D(\qq_k^*\|\tq_k) - H(\qq_k^*\|\tq_k) \le D(\qq_k^*\|\tq_k) \le \eta \varepsilon_k,
\]
which implies $D(\tnu_k^*\|\tnu_k)\le D(\nu_k^*\|\tnu_k) \le \eta \varepsilon_k$ due to the properties of the 
projected point $\tnu_k^*$. To proceed, we use 
the inequality $\log(u) \le u - 1$ that holds for all $u>-1$ to write
\begin{align*}
 D(\nu^*\|\nu_k) -  D(\nu^*\|\tnu_k^*) &= \sum_{x} \nu^*(x) \log\frac{\nu_k(x)}{\tnu_k^*(x)} 
 \le \sum_{x} \nu^*(x) \pa{\frac{\nu_k(x)}{\tnu_k^*(x)}  - 1} = \sum_{x} \frac{\nu^*(x)}{\tnu_k^*(x)} \pa{\nu_k(x)  - 
\tnu_k^*(x)}
\\
&\le \sum_{x} \frac{\nu^*(x)}{\tnu_k^*(x)} \left|\nu_k(x)  - \tnu_k^*(x)\right| \le \max_{x'} 
\frac{\nu^*(x')}{\tnu_k^*(x')} 
\sum_x \left|\nu_k(x)  - \tnu_k^*(x)\right| 
\\
&\le C_\gamma \onenorm{\nu_k - \tnu_k^*} \le C_\gamma \pa{\onenorm{\nu_k - \nu_k^*} + \onenorm{\nu_k^* - \tnu_k} + 
\onenorm{\tnu_k - \tnu_k^*}}
\\
&\qquad\qquad\qquad\mbox{(by Assumption~\ref{ass:conc} and the triangle inequality)}
\\
&\le C_\gamma \pa{\sqrt{\frac{2H(\dd_k^*\|\dd_k)}{1-\gamma}}    +  \sqrt{2D(\nu_k^*\|\tnu_k)}  +  
\sqrt{2D(\tnu_k^*\|\tnu_k)}} 
\\
&\qquad\qquad\qquad\mbox{(by applying Pinsker's inequality twice and invoking Lemma~\ref{lem:qddist})}
\\
& \le C_\gamma \sqrt{\frac{2\alpha \varepsilon_k}{1-\gamma}} + C_\gamma \sqrt{8\eta\varepsilon_k}.
\end{align*}

Plugging all bounds back into the bound of Equation~\eqref{eq:regbound} and using that $C_\gamma \ge 1$, we obtain
\[
\sum_{k=1}^K \iprod{\dd^* - \dd_k}{r} \le \frac{D(\qq^*\|\dd_{0})}{\eta} + \frac{H(\dd^*\|\dd_0)}{\alpha}
 + C_\gamma \pa{\sqrt{\frac{8\alpha}{1-\gamma}} + \sqrt{8\eta}}\sum_{k=1}^K \sqrt{\varepsilon_k}+ \sum_{k=1}^K 
\varepsilon_k, 
\]
thus concluding the proof of the theorem.
\end{proof}

\subsection{The proof of Theorem~\ref{thm:conc}}\label{app:conc}
We will prove the following, more general version of the theorem below:
\setcounter{theorem}{1}
\begin{theorem}(General statement)
Let $\QQ = \ev{Q_\theta: \infnorm{Q_\theta} \le B'}$ for some $B'>0$ and $\Theta$ be the corresponding set of parameter 
vectors, and let $\NN_{\QQ,\epsilon}$ be the $\epsilon$-covering number of $\QQ$ with respect to the $\ell_\infty$ 
norm. Furthermore, define $B = 1+(1+\gamma)B'$, and assume that $\eta B \le 1$ holds. Then, with probability at least 
$1-\delta$, the following holds:
\[
 \sup_{\theta\in\Theta} \left|\wh\GG_k(\theta) - \GG_k(\theta) \right| \le 8 \eta B^2 + 56 
\sqrt{\frac{\log(2\NN_{\QQ,1/\sqrt{N}}/\delta)}{N}}.
\]
\end{theorem}
The proof of the version stated in the main body of the paper follows from bounding the covering number of our linear 
Q-function class as $\NN_{\QQ,\epsilon} \le (1 + 4B/\epsilon)^m$. 

\begin{proof}
We first prove a concentration bound for a fixed $\theta$ and then provide a uniform guarantee through a covering 
argument. 

For the first part, let us fix a confidence level $\delta'>0$ and an arbitrary $\theta$, and define the shorthand 
notation $\wh{S}_n = \wh{\Delta}_\theta(X_{k,n},A_{k,n},X'_{k,n})$ and 
$S_n = \Delta_\theta(X_{k,n},A_{k,n})$. Note that, by definition, these random 
variables are bounded in the interval $[-(\gamma+1)B',1+(\gamma+1)B'] \subset [-B,B]$. 
Furthermore, let us define the notation 
$\EEs{\cdot}{X'} = \EEc{\cdot}{\ev{X_{k,n},A_{k,n}}_{n=1}^N}$ and let
\[
 W = \frac 1N \sum_{n=1}^N e^{\eta \wh{S}_n} \qquad \mbox{and} \qquad \bW = \frac 1N \sum_{n=1}^N e^{\eta 
S_n}.
\]
We start by observing that, by Jensen's inequality, we obviously have $\EEs{W}{X'} \le \bW$. Furthermore, by using 
the inequality $e^u \le 1 + u + u^2$ that holds for all $u\le 1$, we can further write 
\begin{align*}
 \bW &\le \frac 1N \sum_{n=1}^N \pa{1 + \eta S_n + \eta^2 S_n^2} \le 
 \EEs{\frac 1N \sum_{n=1}^N \pa{1 + \eta \wh{S}_n}}{X'}  + \eta^2 S_n^2
 \\
 &\le \EEs{\frac 1N \sum_{n=1}^N e^{\eta \wh{S}_n}}{X'}  + \eta^2 S_n^2 = \EEs{W}{X'} + \eta^2 B^2,
\end{align*}
where in the last line we used the inequality $1+u\le e^u$ that holds for all $u$ and our upper bound on $\wh{S}_n$. 
Thus, taking expectations with respect to $X'$, we get
\begin{equation}\label{eq:WbW}
 \EE{W} \le \EE{\bW} \le \EE{W} + \eta^2 B^2.
\end{equation}

To proceed, we define the function
\[
 f(s_1,s_2,\dots,s_N) = \frac 1N \sum_{n=1}^N e^{\eta s_n}
\]
and notice that it satisfies the bounded-differences property
\[
 f(s_1,s_2,\dots,s_n,\dots,s_N) - f(s_1,s_2,\dots,s_n',\dots,s_N) = \frac 1N \pa{e^{\eta s_n} - e^{\eta 
s_n'}} \le \frac {\eta e^{2\eta B}}N.
\]
Here, the last step follows from Taylor's theorem that implies that there exists a $\chi\in(0,1)$ such that
\[
 e^{\eta s_n'} = e^{\eta s_n} + \eta e^{\eta \chi\pa{s_n' - s_n}}
\]
holds, so that $e^{\eta s_n'} - e^{\eta s_n} = \eta e^{\eta \chi\pa{s_n' - s_n}} \le \eta e^{2\eta B}$, where we used 
the assumption that $|s_n - s_n'| \le 2B$ in the last step. Notice that our assumption $\eta B\le 1$ further implies 
that $e^{2\eta B}\le e^2$. Thus, also noticing that $W = f(S_1,\dots,S_N)$, we can apply McDiarmid's 
inequality that to show that the following holds with probability at least $1-\delta'$:
\begin{equation}\label{eq:Wdiff}
 |W - \EE{W}| \le \eta e^2 \sqrt{\frac{\log(2/\delta')}{N}}.
\end{equation}

Now, let us observe that the difference between the LBE and its empirical counterpart can be written as
\begin{align*}
\wh{\GG}_k(\theta) - \GG_k(\theta) = \frac 1\eta \log\pa{W} - \frac 1\eta \log\pa{\EE{\bW}} = \frac 1\eta 
\log\pa{\frac{W}{\EE{\bW}}}.
\end{align*}
Thus, by combining Equations~\eqref{eq:WbW} and~\eqref{eq:Wdiff}, we obtain that
\begin{align*}
 \wh{\GG}_k(\theta) - \GG_k(\theta) &= \frac 1\eta \log\pa{1 + \frac{W - \EE{\bW}}{\EE{\bW}}} 
 \le \frac 1\eta \log\pa{1 + \frac{W - \EE{W}}{\EE{\bW}}} 
 \\
 &\le 
 \frac{W - \EE{W}}{\eta\EE{\bW}} \le 
e^4 \sqrt{\frac{\log(2/\delta')}{N}},
\end{align*}
where we used the inequality $\log(1+u) \le u$ that holds for $u>-1$ and our assumption on $\eta$ that implies $\bW\ge 
e^{-2}$. 
Similarly, we can show
\begin{align*}
 \GG_k(\theta) - \wh{\GG}_k(\theta) &= \frac 1\eta \log\pa{1 + \frac{\EE{\bW} - W}{W}}
 \le \frac 1\eta \log\pa{1 + \frac{\EE{W} - W + \eta^2 B^2}{W}}
 \\
 &\le 
 \frac{\EE{W} - W + \eta^2 B^2}{\eta W} \le 
e^4 \sqrt{\frac{\log(2/\delta')}{N}} + \eta e^2 B^2,
\end{align*}
This concludes the proof of the concentration result for a fixed $\theta$.

In order to prove a bound that holds uniformly for all values of $\theta$, we will consider a covering of the space of 
Q functions $Q_\theta$ bounded in terms of the supremum norm $\QQ = \ev{Q_\theta:\, \theta\in\real^m, \infnorm{Q_\theta} 
\le B}$. The 
corresponding set of parameters will be denoted as $\Theta$. To define the covering, we fix an $\epsilon > 0$ and 
consider a set $\CC_{\QQ,\epsilon} \subset \QQ$ of minimum cardinality, such that for all $Q_\theta \in \QQ$, there exists a 
$\theta'\in\CC_{\QQ,\epsilon}$ satisfying $|\GG_k(\theta) - \GG_k(\theta')| \le \epsilon$. Defining the covering number 
$\NN_{\QQ,\epsilon} = \left|\CC_{\QQ,\epsilon}\right|$ and $\epsilon = 1/\sqrt{N}$, we can combine the above 
concentration result with a union bound over the covering $\CC_{\QQ,\epsilon}$ to get that
\[
 \sup_{\theta\in\Theta} \left|\GG_k(\theta) - \wh{\GG}_k(\theta)\right| \le \pa{e^4 + 1}
\sqrt{\frac{\log(2\NN_{\QQ,\epsilon}/\delta)}{N}} + \eta e^2 B^2
\]
holds with probability at least $1-\delta$. Upper-bounding the constants $e^2 < 8$ and $e^4+1 < 56$ concludes the proof.
\end{proof}

\subsection{The proof of Proposition~\ref{prop:lossgrad}}\label{app:lossgrad}
For each $i$, the partial derivatives of $S(\theta,z)$ with respect to $\theta_i$ can written as
\begin{align}
\frac{\partial S(\theta,z)}{\partial \theta_i}\sum_n z(n)\frac{\partial  \wh{ 
\Delta}(X_{k,n},A_{k,n},X_{k,n}')}{\partial 
\theta_i} +  \sum_{x,y,a}(1-\gamma)\nu_0(x)\frac{\partial V_\theta(x)}{\partial Q_\theta(y,a)}\frac{\partial 
Q_\theta(y,a)}{\partial \theta_i}.
\label{initialgrad}
\end{align}
Computing the derivatives
\[
\frac{\partial V_\theta(x)}{\partial Q_\theta(y,a)}= \II{x=y} \frac{\pi_k(a|x)e^{\alpha Q_\theta(x,a)}}{\sum_{a'} 
\pi_k(a'|x)e^{\alpha Q_\theta(x,a')}}=  \II{x=y}\pi_{k,\theta}(a|x)
\]
and 
\begin{align*}
\frac{\partial  \wh{ \Delta}(X_{k,n},A_{k,n},X'_{k,n})}{\partial \theta_i}=&\gamma\sum_{x,a}\frac{\partial 
V_\theta(X'_{k,n})}{\partial Q_\theta(x,a)}\frac{\partial Q_\theta(x,a)}{\partial \theta_i}-\frac{\partial 
Q_\theta(X_{k,n},A_{k,n})}{\partial \theta_i}\\
=&\gamma\sum_a\pi_{k,\theta}(X'_{k,n},a)\varphi_i(X'_{k,n},a) - \varphi_i(X_{k,n},A_{k,n})
\end{align*}
and plugging them back in Equation~\eqref{initialgrad}, we get 
\begin{align*}
 \nabla_\theta S(\theta,z)=\sum_{n=1}^N z(n) \pa{\gamma \sum_a \pi_{k,\theta}(a|X_{k,n}) \varphi(X_{k,n}',a) - 
\varphi(X_{k,n},A_{k,n})}  + \sum_{x,a}(1-\gamma)\nu_0(x) \pi_{k,\theta}(a|x) \varphi(x,a).
\end{align*}
The statement of the proposition can now be directly verified using the definitions of $X,A,X'$ and 
$\overline{X},\overline{A}$.
\qed

\newpage
\section{Experimental details and further experiments} \label{app:extra_experiments}

\paragraph{Environment description.}
We use Double-chain and Single-Chain from \citet{furmston2010variational}, River Swim from \citet{strehl2008analysis}, 
WideTree from \citet{ayoub2020model}, CartPole from \citet{brockman2016openai}, Two-State Deterministic from 
\citet{bagnell2003covariant}, windy-grid world from \citet{SB18}, and a new Two-State Stochastic that we present in 
Figure~\ref{fig:two-state-mdp}.

\vspace{-.1cm}
\paragraph{Code environment.} We use the open-source implementation of these algorithms from \citet{Curi2019RLLib} 
which 
is based on PyTorch \citep{paszke2017automatic}. 

\vspace{-.1cm}
\paragraph{Hyperparameters.} In Table~\ref{tab:hyper-param} we show the hyperparameters we use for each environment. 
We fix the regularization parameters as $\eta=\alpha$ and set them so that $1/\eta$ matches the average 
optimal returns in each game.
As optimizers for the player controlling the $\theta$ parameters in \OPT (the learner), we use SGD \citep{RM51} and in 
CartPole we use Adam \citep{kingma2014adam}. For the player controlling the distributions $z$ (the sampler), 
we use the exponentited gradient (EG) update explained in the main text as the default choice, and use the best response (BR) for 
CartPole:
\[
  z_{k,\tau+1}(n) \propto e^{\eta \wh{\Delta}_{k,\tau}(\xi_{k,n})}.
\]
The learning rates $\beta$ and $\beta'$ were picked as the largest values that resulted in stable optimization 
performance. 

\vspace{-.1cm}
\paragraph{Features for CartPole}
We initialize a two-layer neural network with a hidden layer of 200 units and ReLU activations, and use the default 
initialization from PyTorch. We freeze the first layer and use the outputs of the 
activations as state features $\phi':\X\ra \real^{200}$. 
To account for early termination, we multiply each of the features with an indicator feature $\delta(x)$ that takes the 
value $1$ if the transition is valid and $0$ if the next transition terminates. The final state features are given by 
the product $\phi(x) = \phi'(x) \delta(x) \in \mathbb{R}^{200}_{\geq 0}$. Finally, we define state-action features 
$\varphi:\X\times\A\ra \real^{200\times 2}$ by letting $\varphi_{i,b}(x,a) = \phi_i(x) \II{a=b}$ for all $i$ and both 
actions $b\in\A$.

\begin{table}[h]
	\caption{Experiment hyperparameters. The ``-'' symbol indicates that the default values were used, whereas ``x'' 
symbol indicates that the algorithm does not require such hyperparameter.} \label{tab:hyper-param}
    \centering
    \begin{tabular}{l|lllllllll}
        \toprule
        & $\eta$ & $\alpha$ & $\beta$ & $\beta'$ & $\gamma$ & $T$ & Learner & Sampler & Features \\ 
        \midrule 
        Default & 0.5 & 0.5 & 0.1 & 0.1 & 1.0 & 300 & SGD & EG & Tabular\\ 
        Cart Pole & 0.01 & 0.01 & 0.08 & x & 0.99 & - & Adam & BR & Linear\\
        Double Chain & - & - & 0.01 & - & - & -  & - & - & -\\
        River Swim & 2.5 & 2.5 & 0.01 & - & - & - & - & - & -\\
        Single Chain & 5.0 & 5.0 & 0.05 & - & - & - & - & - & -\\
        Two State D & - & - & 0.05 & - & - & - & - & -& -\\
        Two State S & - & - & - & - & - & - & - & -& -\\
        Wide Tree & - & - & - & 0.05 & - & - & - & -& -\\
        Grid World & - & - & - & 0.03 & - & - & - & -& -\\
        \bottomrule
    \end{tabular}
\end{table}

\subsection{The effect of $\eta$ on the bias of the ELBE}
We propose a simple environment to study the magnitude of the bias of the ELBE as an estimator of the LBE. While 
Theorem~\ref{thm:conc} establishes that this bias is of order $\eta$, one may naturally wonder if larger 
values of $\eta$ truly results in larger bias, and if the bias impacts the learning procedure negatively. In this 
section, we show that there indeed exist MDPs where this issue is real.

\begin{figure}[ht]
\centering
\begin{tikzpicture}[auto,node distance=8mm,>=latex,font=\small]

    \tikzstyle{round}=[thick,draw=black,circle]

    \node[round] (s0) {$x_0$};
    \node[round,right=0mm and 20mm of s0] (s1) {$x_1$};
    \node[align=left] at (-0.6,0.75) [red] {$r_\textit{stay}=1$};
    \node[align=left] at (1.3,1) [blue] {$r_\textit{go}=6$};
        \node[align=left] at (2.2,-1.2) {$r_\textit{stochastic}=-3$};

    \draw[->] (s0)  [out=45,in=135, blue]  to (s1) ; 
    \draw[->] (s0) [out=135,in=180,loop, red] to (s0);
    \draw[->] (s1) [out=270,in=315] to (s0);
    \draw[->] (s1) [out=270,in=315,loop]  to (s1);
\end{tikzpicture}
\caption{ Two-state MDP for illustrating the effect of biased estimation of the logistic Bellman error through the 
empirical LBE. From $x_0$ there are two actions with deterministic effects: \textit{stay} and \textit{go}. The 
\emph{stay} action stays in $x_0$ and results in a reward of $1$, while the \emph{go} action moves to 
$x_1$ and results in a reward of $6$. From $x_1$ there is one single stochastic action 
that goes to $x_0$ or remains in $x_1$ with equal probability and has reward 
$-3$.}
\label{fig:two-state-mdp}
\end{figure}
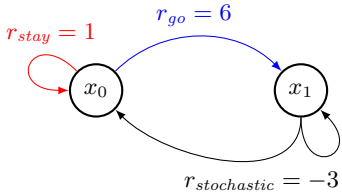

The MDP we consider has two states $x_0$ and $x_1$, with two actions available at $x_0$: \emph{stay} and \emph{go}, 
with the corresponding rewards being $r_{\textit{stay}}$ and $r_{\textit{go}}$, and the rest of the dynamics is as 
explained on Figure~\ref{fig:two-state-mdp}.
To simplify the reasoning, we set $\gamma=1$ and 
consider the case $r_\textit{stay}=0$ first. In this case, the two policies that systematically pick \textit{stay} and 
\textit{go} respectively would both have zero average reward. Despite this, it can be shown that minimizing the 
empirical LBE in \QREPS converges to a policy that consistently picks the \emph{go} action for any choice of $\eta$. 
This is due to the ``risk-seeking'' effect of the bias in estimating the LBE that favors policies that promise higher 
extreme values of the return. This risk-seeking effect continues to impact the behavior of 
\QREPS even when $r_\textit{stay}=1$ and 
$\eta$ is chosen to be large enough---see the learning curves corresponding to various choices of $\eta$ in 
Figure~\ref{fig:biased_reps}. This suggests that the bias of the LBE can indeed be a concern in practical 
implementations in \QREPS, and that the guidance provided by Theorem~\ref{thm:conc} is essential for tuning 
this hyperparameter.

\begin{figure}[ht]
    \centering
	\includegraphics[scale=1]{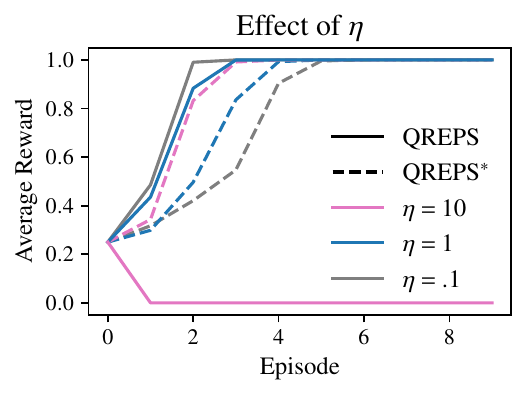}
	\caption{Effect of relative entropy regularization parameter $\eta$ on the performance of \QREPS. 
    On this figure, $\QREPS^*$ (dashed line) refers to the ideal version of the algorithm that minimizes the exact 
LBE, whereas \QREPS (solid line) is the sample-based implementation minimizing the empirical LBE. 
	For large $\eta$, \QREPS suffers from bias and only converges to the optimal policy for smaller values of $\eta$. 
    This effect is independent of the sample size $N$ used for the updates.
    On the other hand, the ideal updates performed by $\QREPS^*$ do not suffer from such bias.}
\label{fig:biased_reps}
\end{figure}

We also note that this bias issue can be alleviated if one has access to a simulator of the environment that allows 
drawing states from the transition distribution $P(\cdot|x,a)$ for any state-action pair in the replay 
buffer\footnote{Note that this condition is relatively mild since it only requires sampling follow-up states for 
state-action pairs that are present in the dataset. In contrast, sampling follow-up states for 
\emph{arbitrary} state-action pairs may be difficult in practical applications where the set of valid 
states may not be known a priori.}. Indeed, in this case one can replace $X'$ by an independently generated sample in 
the gradient estimator $\wh{g}_{k,t}(\theta)$ defined in Equation~\eqref{eq:gradest},
which allows convergence to the 
minimizer of the following semi-empirical 
version of the LBE:
\begin{equation}\label{eq:SELBE}
\begin{split}
 \wt{\dual}_k(\theta) =& \frac1\eta \log \pa{\frac 1N \sum_{n=1}^N e^{\eta 
\Delta_\theta(X_{k,n},A_{k,n})}} +(1-\gamma)\iprod{\nu_0}{\Vtet}.
 \end{split}
\end{equation}
As this definition replaces the empirical Bellman error by the true Bellman error in the exponent, it serves as an 
unbiased estimator of the LBE. Due to this property, one can set large values of the regularization parameter $\eta$ 
and converge faster toward the optimal policy. Thus, this implementation of \QREPS is preferable when one has sampling 
access to the transition function.

\subsection{The Effect of $\alpha$ on the Action Gap}
One interesting feature of the \QREPS optimization problem~\eqref{eq:QREPS_OP} is that it becomes essentially identical 
to the \REPS problem~\eqref{eq:REPS_OP} when setting $\alpha = +\infty$. To see this, let $\Psi$ and $\Phi$ be the 
identity maps so that the primal form of \QREPS becomes
\begin{equation*}
\begin{split}
 \text{maximize}_{\qq,d\in\UU}& \quad \iprod{\qq}{r} - \frac{1}{\eta} D(\qq\|\qq_0)
 \\
 \text{s.t.} \quad & E\transpose d = \gamma P\transpose \qq + (1-\gamma) \nu_0
 \\
 & d = \qq,
\end{split}\raisetag{-1cm}
\end{equation*}
which is clearly seen to be a simple reparametrization of the convex program~\eqref{eq:REPS_OP}. Furthermore, when 
$\alpha = +\infty$, the closed-form expression for $V$ in Proposition~\ref{prop:QREPS-structure} is replaced with the 
inequality constraint $V(x) \ge Q(x,a)$ required to hold for all $x,a$ and the dual function becomes
\[
\mathcal{G}'(Q,V) =\frac1{\eta} \log \left(\sum_{x,a}  {\qq_0(x,a) e^{\eta \pa{r(x,a) + \gamma 
\sum_{x'}P(x'|x,a) V(x') - Q(x,a)}}}\right) + (1-\gamma) \iprod{\nu_0}{V}.
\]
Since this function needs to be minimized in terms of $Q$ and $V$ and it is monotone decreasing in $Q$, its minimum is 
achieved when the constraints are tight and thus when $Q(x,a) = V(x)$ for all $x,a$. Thus, in this case $Q$ loses its 
intuitive interpretation as an action-value function, highlighting the importance of the conditional-entropy 
regularization in making \QREPS practical.

\begin{figure}[t]
\includegraphics[scale=1]{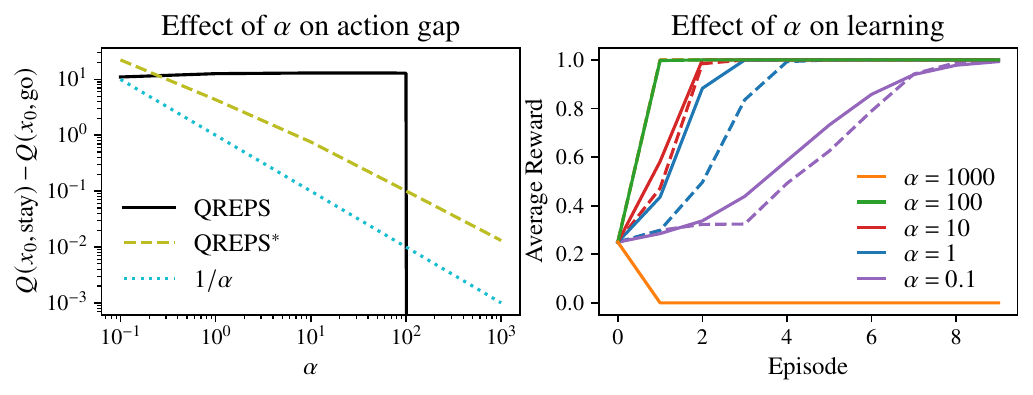}
	\caption{Effect of conditional-entropy regularization parameter $\alpha$ on the performance of \QREPS.
    On this figure, $\QREPS^*$ (dashed line) refers to the ideal version of the algorithm that minimizes the exact 
    LBE, whereas \QREPS (solid line) is the sample-based implementation minimizing the empirical LBE. 
    On the left plot, we see the effect of $\alpha$ on the action gap. 
    For $\QREPS^*$, the action gap decreases at a rate slightly slower than $1/\alpha$. 
    On the other hand, for \QREPS, the estimation noise dominates the action gap for smaller values of $\alpha$. 
    For larger values of $\alpha$, \QREPS fails to identify the optimal action which results in a negative action gap.
    On the right plot, we show the performance for different values of alpha. 
    For $\QREPS^*$, $\alpha$ plays the role of a learning rate: as $\alpha$ increases so does the learning speed.
    For \QREPS, this effect is only preserved for moderate values of $\alpha$, as the small action gap in the ideal 
    Q-values makes identifying the optimal action harder.
    For $\alpha=100$ (green solid line), the sign is identified correctly and it performs almost as if no regularization 
    was present. 
    For $\alpha=1000$ (orange solid line), the sign is misidentified and the wrong action is preferred, leading to poor 
performance.} 
    \label{fig:alg_hyperparams}
    \label{fig:action_gap}
\end{figure}

From a practical perspective, this suggests that the choice of $\alpha$ impacts the gap between the values of $Q$: as 
$\alpha$ goes to infinity, the gap between the values vanish and they become harder to distinguish based on noisy 
observations. Figure~\ref{fig:alg_hyperparams} shows that the action gap indeed decreases as $\alpha$ is increased, 
roughly at an asymptotic rate of $1/\alpha$, and that learning indeed becomes harder as the gaps decrease.

\end{document}